\documentclass{article}

\PassOptionsToPackage{numbers,compress}{natbib}

\usepackage[main, final]{neurips_2026}

\usepackage{amsmath,amsfonts,bm}

\def\eqref#1{(\ref{#1})}

\def\1{\bm{1}}

\def\mI{{\bm{I}}}

\DeclareMathAlphabet{\mathsfit}{\encodingdefault}{\sfdefault}{m}{sl}
\SetMathAlphabet{\mathsfit}{bold}{\encodingdefault}{\sfdefault}{bx}{n}

\newcommand*\lrb[1]{{\left[#1\right]}}

\newcommand*\lrp[1]{{\left(#1\right)}}

\newcommand*\ts[1]{s^{(#1)}}

\newcommand*\ty[1]{y^{(#1)}}

\newcommand*\tb[1]{b^{(#1)}}
\newcommand*\tv[1]{v^{(#1)}}
\newcommand*\tr[1]{r^{(#1)}}
\def\Attn{{\texttt{Attn}}}
\def\Re{{\mathbb{R}}}
\def\pad{{\texttt{pad}}}
\def\TF{{\texttt{TF}}}
\newcommand\numberthis{\addtocounter{equation}{1}\tag{\theequation}}

\usepackage[utf8]{inputenc}  
\usepackage[T1]{fontenc}     
\usepackage{microtype}       
\usepackage{graphicx}
\usepackage{subcaption}
\usepackage{booktabs}        
\usepackage{float}           
\usepackage{placeins}        
\usepackage{hyperref}        
\usepackage{url}             
\usepackage{amsfonts}        
\usepackage{nicefrac}        
\usepackage{xcolor}          
\usepackage{amsmath}
\usepackage{amssymb}
\usepackage{mathtools}
\usepackage{amsthm}
\usepackage[capitalize,noabbrev]{cleveref}
\usepackage[textsize=tiny]{todonotes}

\newcommand{\itxt}[1]{\textit{#1}}
\newcommand{\btxt}[1]{\textbf{#1}}
\newcommand{\itbf}[1]{\textbf{\textit{#1}}}

\newcommand{\cS}{\mathcal{S}}
\newcommand{\cA}{\mathcal{A}}
\newcommand{\cT}{\mathcal{T}}
\newcommand{\cH}{\mathcal{H}}
\newcommand{\rNum}{\mathbb{R}}
\newcommand{\transpose}{\intercal}

\theoremstyle{plain}
\newtheorem{theorem}{Theorem}[section]

\newtheorem{lemma}[theorem]{Lemma}

\theoremstyle{definition}

\theoremstyle{remark}

\let\cite\citet

\title{One for All: A Non-Linear Transformer can Enable\\
  Cross-Domain Generalization for In-Context Reinforcement Learning}

\author{%
  Bowen He\thanks{Corresponding author. Email: \texttt{bowen.he@duke.edu}} \quad
  Juncheng Dong \quad
  Lin Lin \quad
  Xiang Cheng \\
  \\
  Duke University\\
}

\begin{document}

\maketitle

\begin{abstract}
A central challenge in reinforcement learning (RL) is to learn models that generalize beyond the tasks on which they are trained, a goal traditionally pursued through multi-task and meta RL. Recently, transformer architectures have emerged as a promising approach, enabling adaptation to new tasks via in-context learning without explicit parameter updates.  From a functional perspective, a transformer can be viewed as a \btxt{functional operator} that maps a context to a task-specific function. It is thus fundamental to understand and design this operator to support stronger generalization in RL.
In this work, we address this resulting question of generalization from a kernel-based perspective by establishing a connection between non-linear transformers and kernel-based temporal difference learning. By interpreting the transformer as performing regression in a \btxt{Reproducing Kernel Hilbert Space} (RKHS), we show that value functions from different domains can be represented using a shared set of weights, provided they lie within the same RKHS. Experiments on multiple MetaWorld domains support this interpretation, demonstrating convergence of the temporal-difference objective.
\end{abstract}

\section{Introduction}
Conventional reinforcement learning (RL) methods typically learn task-specific policies and value functions that fail to transfer to new tasks, requiring substantial additional training or even learning from scratch. This has motivated the development of task-agnostic RL methods that support generalization across tasks, commonly referred to as multi-task or meta reinforcement learning. A growing line of work explores the use of transformers for this purpose, leveraging their \btxt{in-context learning} capability to implicitly implement adaptation algorithms in their parameters and to condition on task-specific context at inference time. Specifically, in addition to the query state $s_{query}$ at which a decision is required, a transformer also takes as input context information formulated as transitions from an environment $\tau_t \dot{=}(s_0,a_0,r_0,\cdots,s_{t-1},a_{t-1},r_{t-1})$. By conditioning on the context information, the model adapts its behavior for the same $s_{query}$.

While there has been empirical success in the literature, a principled understanding of the mechanisms underlying in-context RL remains limited, particularly with respect to how and why transformers generalize across tasks, and the conditions under which such generalization emerges. This gap hinders the development of more powerful in-context RL methods that generalize not only across tasks within a single domain, but also across domains themselves. Essentially, a transformer can be viewed as a \btxt{functional operator} that generates an implicit function conditioned on the input context. Studying and designing this operator is therefore central to understanding and advancing in-context RL.

In this work, we focus on the policy evaluation setting, where the goal is to estimate the value function of a fixed policy. While policy evaluation alone does not constitute a complete RL procedure---full RL additionally requires policy improvement---it serves as a fundamental subroutine in virtually all modern RL algorithms. Understanding the mechanisms by which in-context models perform policy evaluation is therefore a necessary foundational step toward extending in-context RL to full control. Moreover, focusing on policy evaluation allows us to isolate and rigorously analyze the transformer's ability to learn and generalize value functions, without the additional complexity introduced by policy optimization, leading to clearer theoretical and empirical insights.

One prior work studies this question in a linear setting~\citep{wang2025transformerslearntemporaldifference}, assuming both a linear Markov Reward Process (MRP) and a linear transformer. Under these assumptions, the authors show that a linear transformer can implement gradient descent on the temporal-difference (TD) loss, thereby fitting state-value functions within the forward pass of the transformer layers. However, this construction relies critically on strong linearity assumptions and therefore applies only to tasks that strictly satisfy these conditions, which limits its applicability in more complex and realistic settings.

Building on this view, we approach the problem from a kernel-based perspective. The attention mechanism can be interpreted as implicitly implementing a kernel function~\citep{cheng2024transformersimplementfunctionalgradient}, suggesting that nonlinear transformers perform regression in the corresponding Reproducing Kernel Hilbert Space (RKHS) induced by this kernel. This observation establishes a natural connection between nonlinear transformers and kernel-based RL, thereby relaxing the restrictive linearity assumptions on the tasks as well. More importantly, it implies that tasks whose value functions lie in the same RKHS can be addressed by a single transformer, even when they originate from different domains. We validate this hypothesis by applying a transformer with a fixed set of weights across multiple MetaWorld domains, demonstrating convergence of the TD loss. To our knowledge, we provide the first theoretical construction showing that in-context RL with transformers can generalize across domains, offering a foundation for next steps of developing more universal policies with in-context RL.

\paragraph{Contributions.}
\begin{itemize}
    \item \textbf{Kernel view of in-context policy evaluation.}
    We frame in-context \emph{policy evaluation} through the lens of kernel-based reinforcement learning by leveraging the observation that nonlinear attention induces an RKHS kernel, and use this perspective to relate transformer inference to kernelized TD updates (Section~\ref{sec:theory}).

    \item \textbf{Explicit construction of transformer weights.}
    We give an explicit weight construction showing that a depth-$L$ nonlinear transformer with \emph{fixed} weights can implement $L$ iterations of the kernel TD update purely through its forward pass, using context states as kernel centers (Section~\ref{sec:theory} and Appendix~\ref{app:construction}).

    \item \textbf{A sufficient condition for cross-domain reuse via a shared RKHS.}
    Building on the RKHS interpretation, we explain when cross-domain generalization can be expected to hold and when it may fail (Section~\ref{sec:theory} and Section~\ref{sec:limitations}).

    \item \textbf{Empirical support on real tasks.}
    We instantiate the theoretical construction in practice by applying the transformer to MetaWorld tasks~\citep{yu2021metaworldbenchmarkevaluationmultitask}, where the empirical results resonate with the theoretical predictions. (Section~\ref{sec:experiments}, Appendix~\ref{app:learning_curve} and Appendix~\ref{app:one_for_all})

    \item \textbf{Transfer analysis.}
    We evaluate whether a model trained on a single domain can be applied to other domains (Appendix~\ref{app:one_for_all}) and relate its success or failure to kernel-parameter mismatch, consistent with the fixed-kernel limitation discussed in Section~\ref{sec:limitations}. We further analyze a practical failure mode arising from discontinuities in the induced value landscape (Appendix~\ref{app:discontinuous_state_value}).
\end{itemize}

\section{Related Works}
\label{sec:related_works}
\btxt{In-context learning.} The concept of in-context learning~\citep{xie2022explanationincontextlearningimplicit} was introduced to describe the phenomenon whereby pretrained large language models adapt their behavior by conditioning on a prompt, referred to as the context. Since its introduction, in-context learning has attracted substantial attention from the research community~\citep{ICL:akyürek2023learningalgorithmincontextlearning,ICL:dai2023gptlearnincontextlanguage,ICL:garg2023transformerslearnincontextcase,ICL:hahn2023theory,ICL:wang2024largelanguagemodelslatent}, as it offers a principled mechanism for realizing meta-learning: rather than representing a fixed function, the model implicitly implements an algorithm that adapts its behavior based on the input context.

To further understand the underlying mechanisms,~\cite{vonoswald2023transformerslearnincontextgradient} consider a linear setting and show that linear transformers can implement a family of gradient-descent algorithms, implicitly optimizing a function in their forward pass. This weight construction was later theoretically supported by subsequent work~\citep{ahn2023transformerslearnimplementpreconditioned,zhang2023trainedtransformerslearnlinear,mahankali2023stepgradientdescentprovably}, which shows that it closely matches the global optimum of the in-context loss for a one-layer transformer and that training such models converges to this optimum in polynomial time.

While the above works focus on linear transformers, another line of research interprets the attention mechanism as a kernel operation and designs attention computations based on different kernel choices~\citep{tsai2019transformerdissectionunifiedunderstanding,choromanski2022rethinkingattentionperformers,elnouby2021xcitcrosscovarianceimagetransformers,nguyen2022improvingtransformersprobabilisticattention}. Building on this perspective,~\cite{cheng2024transformersimplementfunctionalgradient} further demonstrates that nonlinear transformers can perform kernel regression with induced kernel functions, providing a principled explanation of their behavior.

\btxt{In-context reinforcement learning.} In-context reinforcement learning (RL) introduces additional complexity, as reinforcement learning algorithms typically involve dynamic programming procedures for policy evaluation and policy improvement~\citep{sutton1998reinforcement}. To isolate this complexity, prior empirical work pretrains transformers in a supervised manner for decision-making tasks~\citep{laskin2022incontextreinforcementlearningalgorithm,lee2023supervisedpretraininglearnincontext}, demonstrating that such pretrained models can perform effectively on tasks not explicitly encountered during training. Building on this idea, subsequent methods modify the structure of the context information to provide more informative trajectories, enabling transformers to implement policy improvement within their forward pass~\citep{shi2023crossepisodiccurriculumtransformeragents,huang2024incontextdecisiontransformerreinforcement,liu2023emergentagentictransformerchain}. \citet{tarasov2026yes} further show that combining the in-context RL objective with a Q-learning objective yields direct gains.

On the theoretical side, recent work has begun to formalize the algorithmic capabilities of transformers in in-context RL. \citet{lin2024transformersdecisionmakersprovable} provide a provable explanation for supervised pretraining-based approaches, showing that transformers can implement classical reinforcement learning algorithms such as UCB and Thompson sampling through their forward pass. Complementarily, \citet{wang2025transformerslearntemporaldifference} study policy evaluation in the linear setting, demonstrating that transformers can implicitly implement TD algorithms in their learned weights.  As in-context RL is a rapidly evolving area encompassing a wide range of paradigms and perspectives, we refer readers to the survey by \citet{moeini2025surveyincontextreinforcementlearning} for a more comprehensive overview.

Toward the goal of realizing a universal policy through in-context RL, \citet{kirsch2023towards} provide initial empirical evidence that a transformer pretrained on tasks from a single domain, specifically Ant in their experiments, can perform in-context adaptation to other domains such as Reacher, HalfCheetah, and CartPole. These results support the hypothesis that a transformer with a fixed set of weights can learn in context across multiple domains; however, the work does not offer a formal justification for this behavior. In this work, we take a first step toward providing such an explanation, shedding light on the development of more general and transferable policy models.

\btxt{Kernel-based reinforcement learning.} Kernel-based RL has a long history as a means of replacing the linear function approximation for value estimation, enabling flexible state representations and theoretical guarantees under certain assumptions~\citep{Ormoneit2002,engel2005reinforcementlearninggaussianprocesses}. A body of work studies kernelized TD learning for policy evaluation, including both batch and online variants, which embed value functions in a RKHS to achieve consistent estimation in continuous state spaces~\citep{10.1109/TNN.2007.899161,10.1145/1553374.1553504}. These methods provide principled alternatives to parametric function approximation by explicitly introducing nonlinear function classes for state-value estimation.

\section{Settings}
\btxt{Reinforcement Learning.}
We formulate decision making as a Markov Decision Process (MDP)
$(\cS,\cA,\cT,R,p_0,\gamma)$, where $\cS$ and $\cA$ denote the state and action spaces, $\cT$ the transition dynamics, $R$ the reward function, $p_0$ the initial state distribution, and $\gamma$ the discount factor~\citep{sutton1998reinforcement}. In continuous setting, $\cS\subseteq\rNum^{|\cS|}$ and $\cA\subseteq\rNum^{|\cA|}$, with $\cT:\cS\times\cA\to\Delta(\cS)$. For a policy $\pi$, the state-value function is
\[
V^\pi(s)=
\mathbb{E}\!\left[
\sum_{t=0}^{\infty}\gamma^t R(s_t,a_t)
\,\middle|\, s_0=s
\right],
\]
where the expectation is taken over both $\cT$ and $\pi$.

Fixing $\pi$ reduces the MDP to a Markov Reward Process (MRP)
$(\cS,\cT^\pi,R^\pi,p_0,\gamma)$, where
\[
\cT^\pi(s'\mid s)=\int \pi(a\mid s)\cT(s'\mid s,a)\,da,
\qquad
R^\pi(s)=\int \pi(a\mid s)R(s,a)\,da .
\]
Correspondingly, transitions up to time step $t$ in the MDP, 
$(s_0, a_0, r_0, \ldots, s_{t-1}, a_{t-1}, r_{t-1})$, 
is reduced to $(s_0, r_0, \ldots, s_{t-1}, r_{t-1})$, and the state-value function can be written as
\[
V(s) = \mathbb{E}\!\left[ \sum_{t=0}^{\infty} \gamma^t R^{\pi}(s_t) \,\middle|\, s_0 = s \right],
\]
where the expectation is taken only with respect to the transition dynamics $\cT^{\pi}$. The goal of policy evaluation is therefore to compute the state-value function $V(s)$ with interaction trajectories $(s_0, r_0, \ldots, s_{t-1}, r_{t-1})$.

Linear policy evaluation approximates the value function using a linear function of features,
\[
V(s) \approx w^{\transpose}\phi(s),
\]
where $\phi(s)$ is a predefined or learned feature map that embeds a state into the feature space.  
Kernel-based policy evaluation instead models the value function in an RKHS
$\cH_\kappa$ induced by a kernel $\kappa$:
\[
V(s)=\langle V,\kappa(s,\cdot)\rangle_{\cH_\kappa}
\approx \sum_{i=1}^{n}\alpha_i\kappa(s^{(i)},s).
\]
Given the current iterate $V^{(m)}$, we define the empirical TD error as
\[
\tb{m}_t
=
r_t+\gamma V^{(m)}(s_{t+1})-V^{(m)}(s_t).
\]
The kernelized semi-gradient TD update is thus
\[
V^{(m+1)}(\cdot)
=
V^{(m)}(\cdot)
+
\alpha_m
\sum_{t=0}^{n-1}
\tb{m}_t\,\kappa(s_t,\cdot),
\]
which treats the bootstrapped target
$r_t+\gamma V^{(m)}(s_{t+1})$ as fixed within each update. Thus, this is
the RKHS analogue of classical semi-gradient TD, rather than exact
gradient descent on the full squared Bellman residual.

We show that a nonlinear transformer can implement this kernelized TD
update in context: each layer performs one TD-style update over the
in-context transitions, and the multi-layer forward pass carries out
iterative RKHS-based policy evaluation.

\btxt{Transformer.}
We define attention as
\[
\Attn^{\tilde h}_{K,Q,V}(Z)
\dot{=}
VZ\cdot M\cdot \tilde h(KZ,QZ),
\]
where $Z\in\rNum^{(2d+1)\times(n+1)}$ is the input, $K,Q,V$ are projection matrices, and
\[
\tilde h:
\rNum^{(2d+1)\times(n+1)}
\times
\rNum^{(2d+1)\times(n+1)}
\to
\rNum^{(n+1)\times(n+1)}
\]
is the nonlinear activation used in attention. The matrix $M$ is a masking  given by
\[
    M \dot{=} 
    \begin{bmatrix}
        \mI_n & \mathbf{0}_{n\times1}\\
        \mathbf{0}_{1\times n} & 0        
    \end{bmatrix},
    \numberthis \label{e:M}
\]
which prevents information flow from the query column to the context during attention.

Correspondingly, a transformer layer is defined via a residual update,
\begin{equation}
\label{eq:oneHeadTransformer}
    Z_{l+1} \dot{=} Z_l + \Attn^{\tilde{h}}_{K_l, Q_l, V_l}(Z_l),
\end{equation}
and its multi-head variant is given by
\begin{equation}
\label{eq:multiHeadTransformer}
    Z_{l+1} \dot{=} Z_l + \sum_{i=1}^H \Attn^{\tilde{h}}_{K_l^i, Q_l^i, V_l^i}(Z_l).
\end{equation}
For a $k$-layer transformer, either~\ref{eq:oneHeadTransformer} or~\ref{eq:multiHeadTransformer} is applied iteratively to update the input matrix layer by layer. We denote the  output matrices $Z_l$ of the transformer as the output
\[
    \text{TF}_l^\theta(Z_0) \dot{=} \, Z_l, \quad l \in \{1, \ldots, k\},
\]
where $\theta$ denotes the collection of parameters of the transformer.

\btxt{Nonlinear Activation.}
Following~\citet{cheng2024transformersimplementfunctionalgradient}, a nonlinear attention activation can be interpreted as inducing a kernel over input vectors. With linear activation,
\[
[\tilde h(U,V)]_{i,j}\dot{=}U_i^\transpose V_j,
\]
the induced kernel is linear:
\[
\kappa(x,y)=x^\transpose y.
\]
With exponential activation,
\[
[\tilde h(U,V)]_{i,j}
\dot{=}
\exp\!\left(\frac{U_i^\transpose V_j}{\delta}\right),
\numberthis \label{e:th}
\]
the induced kernel becomes
\[
\kappa(x,y)
=
\exp\!\left(\frac{x^\transpose y}{\delta}\right),
\numberthis \label{e:kappa}
\]
where $\delta$ is a temperature parameter. This kernel perspective allows us to view nonlinear transformers with fixed weights as performing kernel regression in their forward pass. Building on this interpretation, we connect nonlinear transformers to kernel-based policy evaluation and show how a single set of weights can operate across multiple domains.

\begin{figure*}[t]
    \centering
    \begin{subfigure}{0.19\textwidth}
        \centering
        \includegraphics[width=\textwidth]{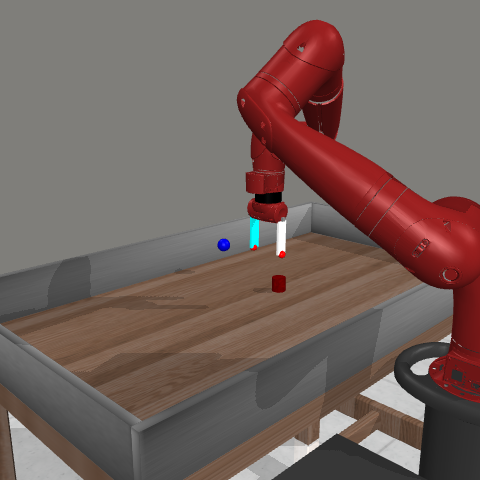}
    \end{subfigure}
    \begin{subfigure}{0.19\textwidth}
        \centering
        \includegraphics[width=\textwidth]{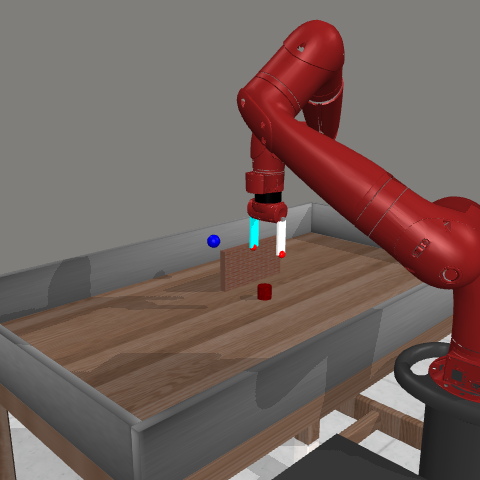}
    \end{subfigure}
    \begin{subfigure}{0.19\textwidth}
        \centering
        \includegraphics[width=\textwidth]{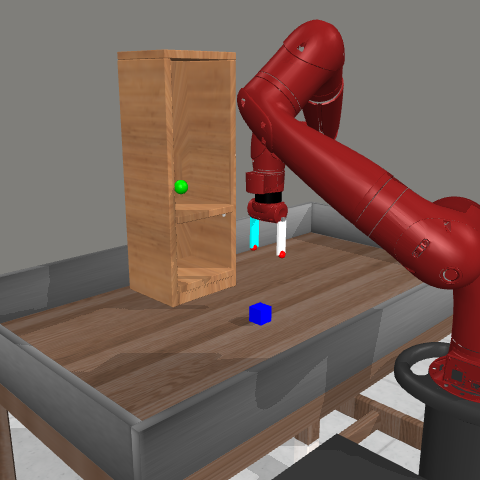}
    \end{subfigure}
    \begin{subfigure}{0.19\textwidth}
        \centering
        \includegraphics[width=\textwidth]{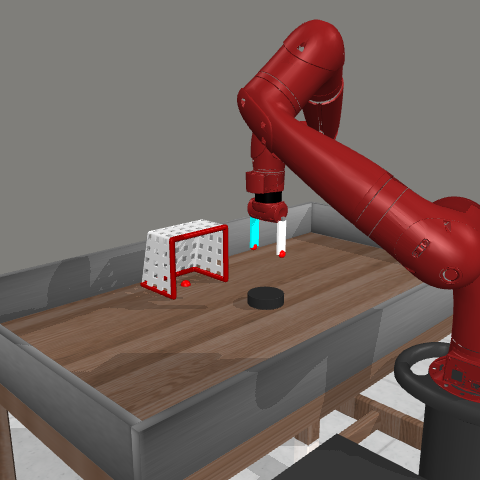}
    \end{subfigure}
    \begin{subfigure}{0.19\textwidth}
        \centering
        \includegraphics[width=\textwidth]{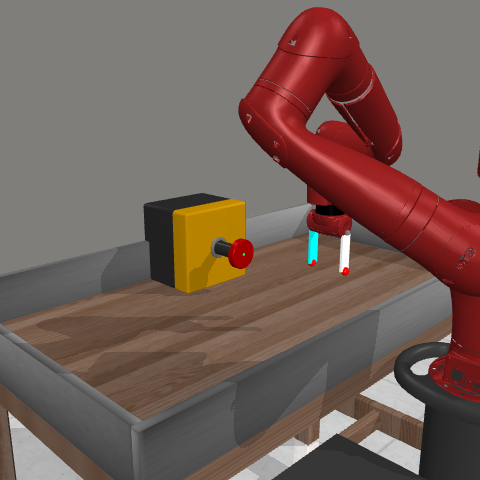}
    \end{subfigure}
    \caption{Task domains from MetaWorld. From left to right: Pick-Place, Pick-Place-Wall, Shelf-Place, Plate-Slide, and Button-Press.}
    \label{fig:env_figures}
\end{figure*}

\section{The In-Context Kernel TD Update}
\label{sec:theory}
\paragraph{Notation.}
We use $i,j$ to denote the \emph{trajectory} time index, and $k,\ell$ to denote the iteration index of the TD algorithm.
Let $\ts{i}\in \Re^d$ denote the state vector at trajectory index $i$ and let $\tr{i}\in\Re$ denote the reward at time $i$.
Let $\tv{i}_k$ denote the value of state $\ts{i}$ after $k$ iterations of the TD update.
Define the TD residual
\[
\tb{i}_k \;:=\; \tr{i} + \gamma \tv{i+1}_k - \tv{i}_k.
\]
We assume the kernel $\kappa$ is symmetric, i.e.\ $\kappa(x,y)=\kappa(y,x)$.

\paragraph{Kernel TD update.}
One step of the kernel TD update is
\begin{align}
    \tv{i}_{k+1}
    = \tv{i}_k + \alpha_k \sum_{j=0}^{n-1} \tb{j}_k \kappa(\ts{j}, \ts{i}).
    \label{e:linear_td_standard}
\end{align}

\paragraph{Equivalent update for $\tb{i}_k$.}
Using the definition of $\tb{i}_k$ and \eqref{e:linear_td_standard}, we obtain
\begin{alignat*}{1}
    \tb{i}_{k+1}
    &= \tb{i}_k + \gamma \lrp{\tv{i+1}_{k+1} - \tv{i+1}_k} - \lrp{\tv{i}_{k+1} - \tv{i}_k}\numberthis \label{e:tb}\\
    &= \tb{i}_k
    + \underbrace{\gamma \alpha_k \sum_{j=0}^{n-1} \tb{j}_k \kappa(\ts{j}, \ts{i+1})}_{\Attn^2}
    - \underbrace{\alpha_k \sum_{j=0}^{n-1} \tb{j}_k \kappa(\ts{j},\ts{i})}_{\Attn^1}.
\end{alignat*}

\paragraph{Initialization.}
We initialize $\tv{i}_0 = 0$ for all $i$. Consequently, $\tb{i}_0=\tr{i}$ for $i=0,\dots,n-1$.
For the query column we set $\tb{n}_0=0$.
We will construct a self-attention layer with two heads (denoted $\Attn^1$ and $\Attn^2$) which exactly compute the
two underbraced expressions in \eqref{e:tb}.

\subsection{Main Theoretical Result}

Let $Z_\ell \in \Re^{(2d+1)\times(n+1)}$ denote the output after $\ell$ Transformer layers.
We will show that a single fixed set of Transformer weights can implement $\ell$ steps of the TD iteration
\eqref{e:linear_td_standard} \emph{in-context}: the weights are universal, while the input matrix $Z_0$ encodes
the particular context states and the query state.

Let $\ts{\pad}$ denote a fixed padding state (can be taken to be $0$). For each TD iteration index $k$, recall that
$\tv{i}_k$ denotes the value of state $\ts{i}$ after $k$ iterations of TD; in particular,
$\tv{\pad}_k$ denotes the value of the padding state $\ts{\pad}$ after $k$ TD iterations.
For each $\ell\in\{0,1,\dots,L\}$, define
\begin{align}
    Z_\ell \;=\;
\begin{bmatrix}
    \ts{0} & \cdots & \ts{n-1} & \ts{n} \\
    \ts{1} & \cdots & \ts{n}   & \ts{\pad} \\
    \tb{0}_\ell & \cdots & \tb{n-1}_\ell & -\tv{n}_\ell + \gamma \tv{\pad}_\ell
\end{bmatrix}.
\label{e:zl}
\end{align}

Here the scalars $\{\tb{i}_\ell\}_{i=0}^{n-1}$ and $\tv{n}_\ell$ are the iterates produced by the $\ell$-step TD update
\eqref{e:linear_td_standard} (with kernel $\kappa$ defined in \eqref{e:kappa}) when applied to the trajectory
$(\ts{0},\dots,\ts{n})$.

For ease of discussion, we explicitly specify the initialization
\begin{align}
    Z_0 \;=\;
    \begin{bmatrix}
        \ts{0} & \cdots & \ts{n-1} & \ts{n} \\
        \ts{1} & \cdots & \ts{n}   & \ts{\pad} \\
        \tr{0} & \cdots & \tr{n-1} & 0
    \end{bmatrix}.
\label{e:z0}
\end{align}

\begin{theorem}[In-context implementation of $L$ steps of TD]
\label{thm:tf-td}
Fix $n$ and let $\ts{\pad}$ be a padding state. Consider the TD iteration \eqref{e:linear_td_standard}
with kernel $\kappa$ defined in \eqref{e:kappa}.
There exist Transformer weights
\[
\theta \;=\; \Bigl\{(K_{\ell}^{j},Q_{\ell}^{j},V_{\ell}^{j}) : j\in\{0,1\},\ \ell\in\{1,\dots,L\}\Bigr\},
\]
independent of the particular trajectory, such that the following holds:
for any trajectory $(\ts{0},\tr{0},\dots,\ts{n-1},\tr{n-1},\ts{n})$, if $Z_0$ is defined as in \eqref{e:z0}, then
\[
\TF^{\theta}_\ell(Z_0) \;=\; Z_\ell \qquad \text{for all }\ \ell\in\{1,\dots,L\},
\]
where $Z_\ell$ is as defined in \eqref{e:zl}. In particular, the last row of $Z_\ell$ contains the $\ell^{th}$ iteration TD values for each state.
\end{theorem}

\paragraph{Discussion.}
The transformer uses states from the context transitions as kernel centers for the TD update and directly evaluates the query state $\ts{n}$ to obtain its value. The value can be read out directly as the negative of the bottom-right entry of the matrix $Z_l$. In this sense, the transformer is explicitly equipped with a TD algorithm for estimating state values and primarily requires transition data from a new MRP to perform this computation. However, the current procedure introduces a constant bias that is independent of the query state; we defer a discussion of this bias to Section~\ref{sec:limitations}. Nevertheless, this formulation establishes a principled connection between non-linear transformers and kernelized reinforcement learning, thereby enabling analysis through the lens of kernel-based RL.
We provide the \itbf{explicit constructions of the the two heads} in appendix~\ref{app:construction} with additional details to analyze the computation.

\begin{proof}[Proof of Theorem~\ref{thm:tf-td}]
We explicitly construct the weights $\theta$ by specifying, for each layer $\ell\in\{0,1,\dots,L-1\}$, a
two-head attention block whose first head implements the term $\Attn^1$ and whose second head implements the term
$\Attn^2$ in \eqref{e:tb}. Concretely, in layer $\ell+1$ we set
\[
(K_{\ell+1}^{1},Q_{\ell+1}^{1},V_{\ell+1}^{1}) \text{ to be the matrices in Lemma~\ref{lem:attn1},}
\]
\[
(K_{\ell+1}^{2},Q_{\ell+1}^{2},V_{\ell+1}^{2}) \text{ to be the matrices in Lemma~\ref{lem:attn2}.}
\]

We prove by induction on $\ell$ that the layer outputs satisfy $\TF_\ell^\theta(Z_0)=Z_\ell$, where $Z_\ell$ is defined
in \eqref{e:zl}.

\paragraph{Base case.}
By construction, the Transformer input is initialized as $Z_0$ in \eqref{e:z0}. This matches \eqref{e:zl} at $\ell=0$
since $\tv{i}_0=0$ for all $i$ and thus $\tb{i}_0=\tr{i}$ for $i=0,\dots,n-1$, with the query-column residual set to
$\tb{n}_0=0$.

\paragraph{Inductive step.}
Assume for some $\ell\in\{0,1,\dots,L-1\}$ that $\TF_\ell^\theta(Z_0)=Z_\ell$.
Apply the $(\ell+1)$-st Transformer layer. By Lemma~\ref{lem:attn1}, head $1$ produces an update matrix whose only
nonzero entries lie in the last row and equal
\[
-\alpha_\ell \sum_{j=0}^{n-1} \tb{j}_\ell \kappa(\ts{j},\ts{i})
\quad \text{in column } i.
\]
By Lemma~\ref{lem:attn2}, head $2$ produces an update matrix whose only nonzero entries lie in the last row and equal
\[
\gamma\alpha_\ell \sum_{j=0}^{n-1} \tb{j}_\ell \kappa(\ts{j},\ts{i+1})
\quad \text{in column } i\in\{0,\dots,n-1\},
\text{and}
\]
\[
\gamma\alpha_\ell \sum_{j=0}^{n-1} \tb{j}_\ell \kappa(\ts{j},\ts{\pad})
\quad \text{in column } n.
\]
Summing the two heads and adding the result to the current last row (as per the Transformer layer definition) yields,
for each $i\in\{0,\dots,n-1\}$,
\[
\tb{i}_{\ell+1}
=
\tb{i}_\ell
+ \alpha_\ell
\sum_{j=0}^{n-1} \tb{j}_\ell
\Big(
\gamma\,\kappa(\ts{j}, \ts{i+1})
- \kappa(\ts{j}, \ts{i})
\Big).
\]
which is exactly the TD residual update \eqref{e:tb}. The first two block-rows are unchanged by our choice of $V$ in both
heads, hence remain $(\ts{0},\dots,\ts{n})$ and $(\ts{1},\dots,\ts{n},\ts{\pad})$.
Finally, the query column update equals the same expression with $\ts{i+1}$ replaced by $\ts{\pad}$, giving precisely the
form $\tv{n}_{\ell+1} + \gamma \tv{\pad}_{\ell+1}$ in \eqref{e:zl}.

Therefore the layer output equals $Z_{\ell+1}$, completing the inductive step.

By induction, $\TF_\ell^\theta(Z_0)=Z_\ell$ for all $\ell\in\{1,\dots,L\}$. This conclude the proof.
\end{proof}

\section{Limitations}
\label{sec:limitations}

While our construction establishes a principled connection between transformers and kernel-based TD learning, several important limitations remain that warrant discussion and suggest promising directions for future work.

\btxt{Constant offset in value estimation.}
The current construction introduces a state-independent offset to all value estimates, specifically
\[
    V(\ts{n}) = \tv{n}_\ell - \gamma \tv{\pad}_\ell.
\]
This offset implies that the absolute value estimates contain a constant bias that does not vanish even when the underlying value function is perfectly represented within the RKHS. However, this limitation has restricted practical impact for two key reasons. First, in policy improvement procedures such as policy gradient methods or advantage-based algorithms, only \emph{relative} value comparisons between states or state-action pairs matter. The constant offset cancels when computing advantages or comparing action values, leaving policy learning unaffected. Second, when computing TD errors during training, the offset appears in both $V(s_t)$ and $V(s_{t+1})$, causing it to cancel in the TD error. This ensures that the TD loss converges to zero during training despite the absolute value bias, as we empirically demonstrate in Section~\ref{sec:experiments} about limitations.

Nevertheless, addressing this offset through architectural modifications---such as introducing additional normalization layers or learnable bias terms---represents a natural avenue for improving the construction's accuracy for applications requiring absolute value estimates.

\btxt{Fixed kernel parameters.}
A more fundamental limitation is that the kernel function parameters, such as the temperature $\delta$ in the exponential kernel $\kappa(x,y) = \exp(x^\transpose y / \delta)$, remain fixed throughout the forward pass and are not adapted based on the input context. This restricts the transformer's ability to handle tasks whose value functions require fundamentally different kernel characteristics. For instance, consider two tasks: one requiring smooth, slowly-varying value functions (suggesting a large kernel bandwidth) and another exhibiting sharp, localized value structure (requiring a small bandwidth). A single fixed kernel cannot optimally represent both.

From the perspective of our RKHS interpretation, this limitation can be understood as follows. Each choice of kernel $\kappa$ induces a distinct RKHS $\mathcal{H}_\kappa$ with its own native norm and function characteristics. Our construction guarantees accurate value function approximation only when $V^\star \in \mathcal{H}_\kappa$, i.e., when the true value function lies within the RKHS induced by the chosen kernel. If two tasks induce value functions $V_1^\star$ and $V_2^\star$ that lie in different RKHSs, then no single fixed kernel can represent both functions well.

This naturally raises the question of whether kernel parameters can be adapted in context, dynamically adjusting the induced RKHS based on contextual and task-specific characteristics, or whether it is possible to learn the kernel itself in context. We leave the investigation of these directions for future work.

Despite these limitations, our construction and analysis provide a justification for why transformers with fixed weights can perform policy evaluation across multiple domains, laying the groundwork for developing more sophisticated in-context RL methods with principled generalization guarantees.

\section{Experiments}
\label{sec:experiments}
MetaWorld~\citep{yu2021metaworldbenchmarkevaluationmultitask} provides an ideal testbed for analyzing multi-task RL methods. It features a diverse set of environments, each consisting of 50 task variations with different goal configurations. These environments can be viewed as distinct domains, as they require qualitatively different manipulation behaviors from the robot arm, posing a significant challenge for generalization. In addition, MetaWorld provides implementations of near-optimal policies for each domain. Accordingly, for each task within a domain, we construct the corresponding MRPs by injecting noise of different levels into the actions generated by the associated policy.

\subsection{Implementation}
\label{sec:implementation}
Our experiments are designed to validate the theoretical in-context TD mechanism rather than to train a generic transformer end-to-end. Accordingly, we instantiate the attention weights exactly as constructed in Appendix~\ref{app:construction}. We stack multiple transformer layers, each constructed to perform policy evaluation. During training, we introduce a learnable scalar parameter $\alpha$ and a constant $n$ to modulate the residual update, given by
\[
    Z_{l+1} \dot{=} Z_l + \frac{\alpha}{n} \sum_{i=1}^2
    \Attn^{\tilde{h}}_{K_l^i, Q_l^i, V_l^i}(Z_l).
\]
Intuitively, $\alpha$ corresponds to the effective learning rate of the kernel regression implicitly implemented in the forward pass and therefore must be learned. The constant $n$ denotes the context length and is introduced to stabilize training. In our experiments, we adopt the exponential activation and the corresponding exponential kernel in the transformer layers, as this choice yields an exact equivalence between the attention computation and the induced kernel. The softmax activation is closely related to the exponential kernel, differing only by an additional normalization factor in the underlying kernel function, as noted by~\citep{cheng2024transformersimplementfunctionalgradient}. The choice between exponential and softmax activations effectively determines the induced RKHS.

\subsection{One-Domain Generalization}
We select the following domains to evaluate our transformer: \btxt{Pick-Place-v3}, \btxt{Pick-Place-Wall-v3}, \btxt{Shelf-Place-v3}, \btxt{Plate-Slide-v3}, and \btxt{Button-Press-v3}. Figure~\ref{fig:env_figures} illustrate the different domains. They span a range of difficulty levels and return scales, enabling us to verify the hypothesis underlying the proposed kernel-based perspective. Intuitively, tasks within the same domain are expected to exhibit similar state-value function landscapes, whereas tasks from different domains should induce more distinct function classes. If a transformer with a single set of weights can successfully model value functions across such heterogeneous domains, it would provide strong evidence for cross-domain generalization.

For each task within a domain, we inject noise at \btxt{high}, \btxt{medium}, and \btxt{low} levels to construct distinct MRPs. Table~\ref{tab:mrp_noise} from appendix~\ref{app:noise_level} reports the noise std and the estimated initial state values under different noise levels, demonstrating a coverage of the value range. To train the proposed transformer, we keep the model weights fixed and update only the scaling coefficients of the residual connections between layers. During training, we expose the transformer only to tasks with \btxt{high} noise levels for each domain, while evaluating its TD error on tasks with \btxt{medium} and \btxt{low} noise levels. This setup allows us to assess the in-context learning capability of the model across diverse MRPs, even within the same domain. Since each domain contains \btxt{50} tasks, evaluating three noise levels results in a total of \btxt{150} MRPs. To provide a comprehensive evaluation, at each evaluation point we sample \btxt{5} tasks from the MRPs at different noise levels and report the mean TD error.

Figure~\ref{fig:TD_over_tasks} reports the temporal-difference losses over the entire training process for each domain. \itxt{We restrict value computation to early episode states for practical reasons, as the induced value function is highly discontinuous over the full state space and thus challenging to estimate accurately.} A detailed discussion of this issue is provided in Appendix~\ref{app:discontinuous_state_value}. Nevertheless, the results show consistent decreases in TD loss across all domains, providing evidence that a transformer with a single set of weights can be effectively applied across domains and achieve convergence in terms of TD loss.

\begin{figure}[t]
    \centering
    \setlength{\abovecaptionskip}{2pt}
    \setlength{\belowcaptionskip}{-4pt}

    \begin{subfigure}[t]{0.19\columnwidth}
        \centering
        \includegraphics[width=\linewidth,height=2.0cm,keepaspectratio]{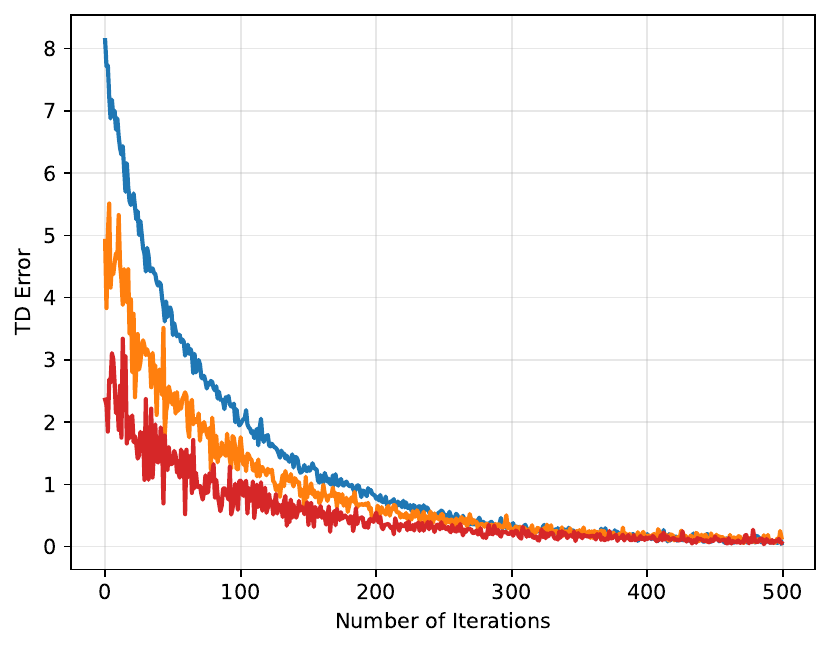}
        \caption{\scriptsize Pick-Place-v3}
    \end{subfigure}
    \hfill
    \begin{subfigure}[t]{0.19\columnwidth}
        \centering
        \includegraphics[width=\linewidth,height=2.0cm,keepaspectratio]{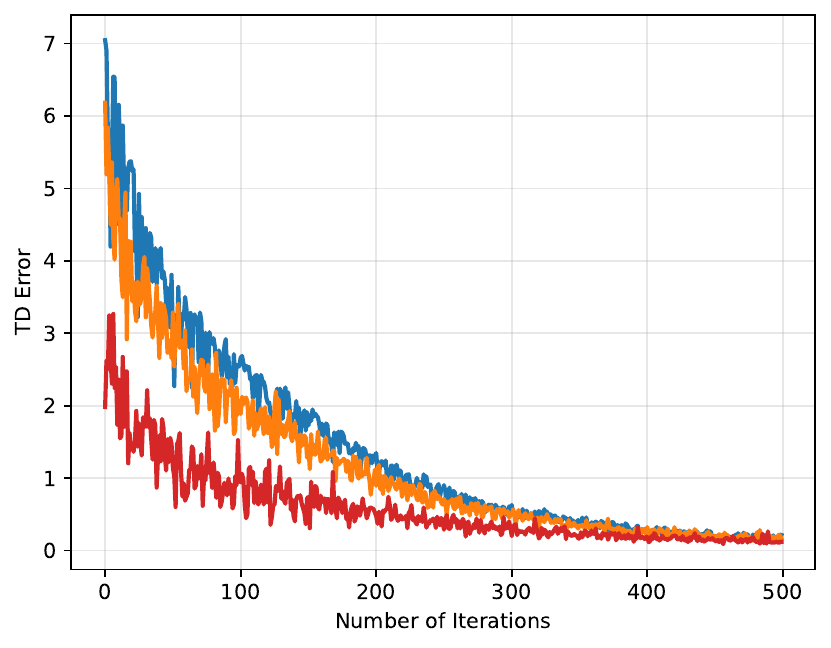}
        \caption{\scriptsize Pick-Place-Wall-v3}
    \end{subfigure}
    \hfill
    \begin{subfigure}[t]{0.19\columnwidth}
        \centering
        \includegraphics[width=\linewidth,height=2.0cm,keepaspectratio]{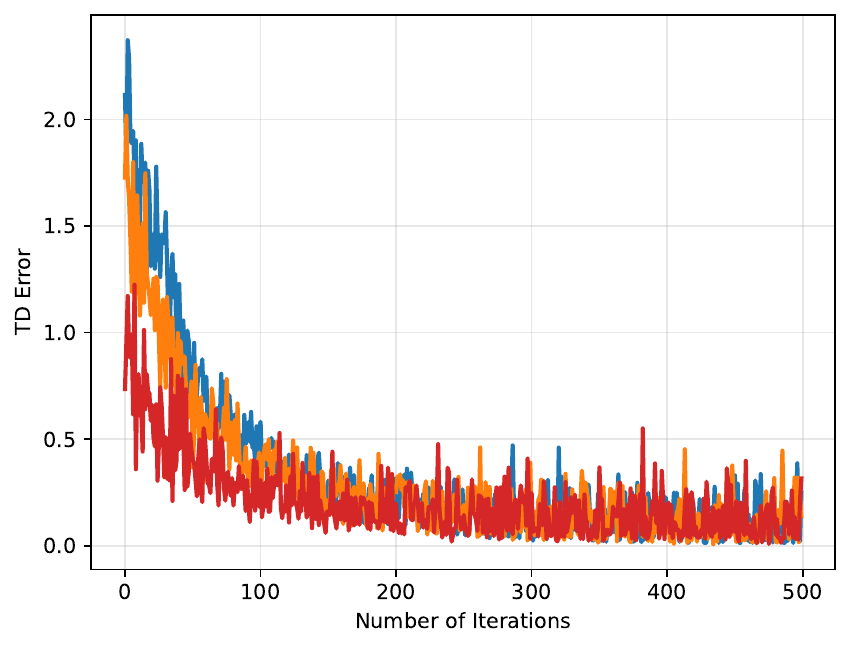}
        \caption{\scriptsize Shelf-Place-v3}
    \end{subfigure}
    \hfill
    \begin{subfigure}[t]{0.19\columnwidth}
        \centering
        \includegraphics[width=\linewidth,height=2.0cm,keepaspectratio]{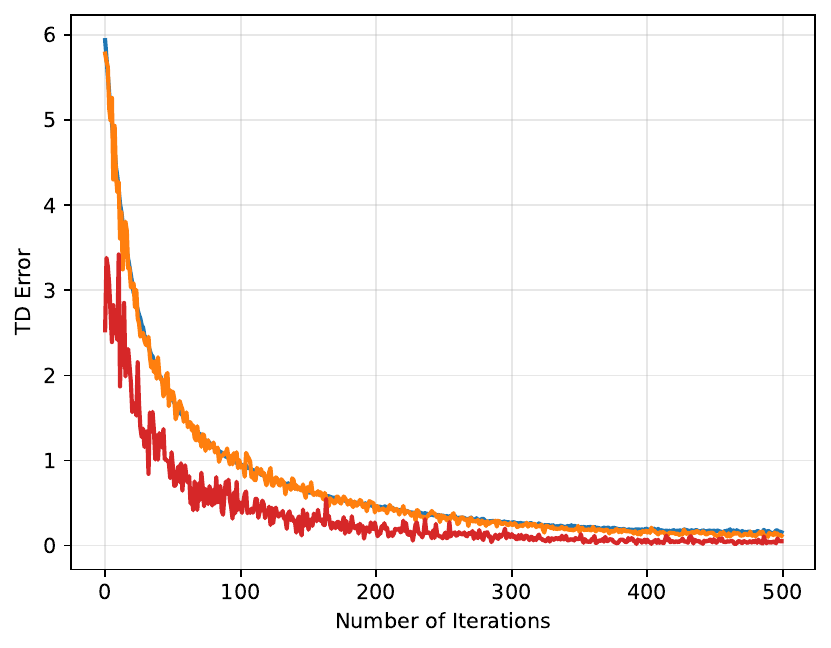}
        \caption{\scriptsize Plate-Slide-v3}
    \end{subfigure}
    \hfill
    \begin{subfigure}[t]{0.19\columnwidth}
        \centering
        \includegraphics[width=\linewidth,height=2.0cm,keepaspectratio]{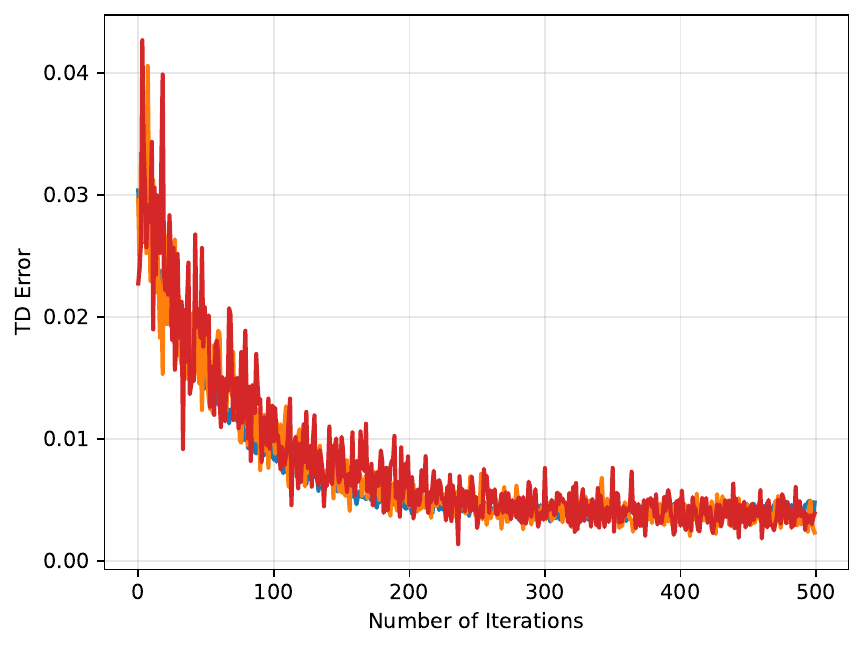}
        \caption{\scriptsize Button-Press-v3}
    \end{subfigure}

    \vspace{1.0em}
    \hspace{30.0em}
    \includegraphics[width=0.19\columnwidth]{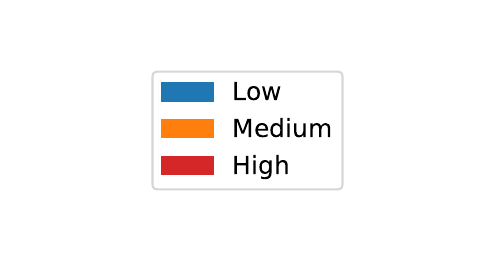}
    \vspace{-0.1em}
    \caption{Cross-task TD errors for multiple model checkpoints during training.}
    \label{fig:TD_over_tasks}
\end{figure}

\paragraph{Linear Transformer Fails.}
To provide additional insight, we report in Appendix~\ref{app:learning_curve} the learning curves of the non-linear transformer alongside those of the linear transformer from~\citet{wang2025transformerslearntemporaldifference}. As expected, the linear transformer performs poorly on MetaWorld tasks, since it can only represent linear functions, which is insufficient to capture the underlying value-function structure.

\subsection{Cross-Domain Generalization}
Another important question is whether the transformer can generalize across different domains, a central question directly reflected in the title of our work. As discussed in Section~\ref{sec:limitations}, we anticipate that such generalization is possible as long as the underlying functions can be well approximated within a single RKHS. To investigate this empirically, we present results in Appendix~\ref{app:one_for_all}, where models trained on one domain at different stages of training are applied to multiple domains. We observe that the models generalize well across Pick-Place-v3, Pick-Place-Wall-v3, Shelf-Place-v3, and Plate-Slide-v3, but fail to transfer to Button-Press-v3. This behavior is consistent with the hyperparameter analysis in Appendix~\ref{app:hyperparameters}, where the four domains share a common kernel parameter, while Button-Press-v3 requires a substantially different setting.

It empirically supports our hypothesis that a single model can generalize not only across tasks within the same domain, but also across different domains.

\subsection{Synthetic Environment}
We design a synthetic MRP whose state value function lies in an RKHS.
The state space of the MRP is two-dimensional, which allows us to directly visualize the learned and ground-truth value functions using three-dimensional surface plots. 
Figure~\ref{fig:syntheic_env} compares the model-approximated state value function with the ground-truth state value function. Specifically, to construct the transformer input $Z_0$, we uniformly sample $32$ states from the state space as $s_0, \ldots, s_{n-1}$, together with their corresponding next states and rewards as context.
We then fix the context and vary the query state across the state space to obtain the predicted state values.
We use $30$ transformer layers, consistent with the setting used in the MetaWorld experiments.

As illustrated in Figure~\ref{fig:syntheic_env}, the transformer accurately captures the overall shape of the state value function.
Although the value ranges differ due to the constant bias term discussed in Section~\ref{sec:limitations}, the relative value comparisons are captured well. This is reflected in the close alignment between the shapes of the model-approximated and ground-truth value functions. Numerically, after accounting for the constant bias in the state-value estimates, the model achieves a mean state-value error as low as 0.00034. We provide additional visualizations, ablation studies, and environment details in Appendix~\ref{app:synthetic_env}. These results further support our hypothesis.

\begin{figure}[t]
    \centering
    \setlength{\abovecaptionskip}{2pt}

    \begin{subfigure}[t]{0.49\columnwidth}
        \centering
        \includegraphics[width=0.9\linewidth,height=3.2cm,keepaspectratio]{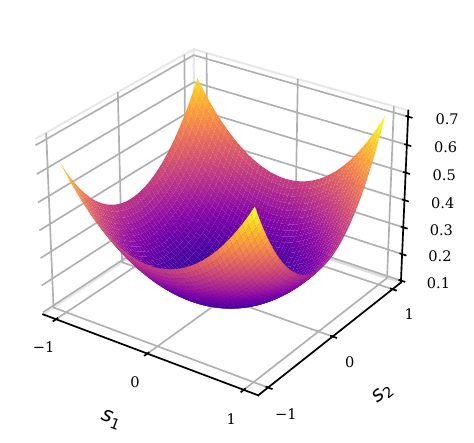}
        \caption{Model approximated state value function.}
        \label{fig:main_synthetic_model_approximated}
    \end{subfigure}
    \hfill
    \begin{subfigure}[t]{0.49\columnwidth}
        \centering
        \includegraphics[width=0.9\linewidth,height=3.2cm,keepaspectratio]{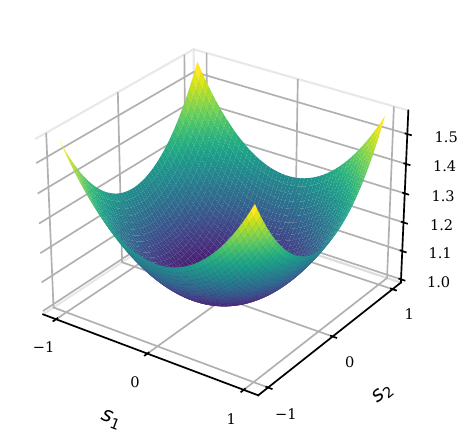}
        \caption{Ground truth state value function.}
        \label{fig:main_synthetic_ground_truth}
    \end{subfigure}

    \caption{Comparison between the model-approximated and ground-truth state value functions.}
    \label{fig:syntheic_env}
\end{figure}

\section{Discussion}
This work provides a first theoretical construction that transformers with non-linear activations can generalize not only across tasks within a single domain, but also across domains in in-context reinforcement learning. By adopting a kernel-based reinforcement learning perspective, we establish a principled connection between non-linear transformers and kernel-based temporal-difference learning, showing that a single transformer with fixed weights can implement policy evaluation across multiple domains when the induced value functions lie within a shared RKHS.

Our analysis assumes a fixed RKHS induced by the transformer's attention mechanism. Under this restriction, performance degrades when the underlying value function lies outside the induced function space. Importantly, this limitation is not merely empirical but follows directly from the RKHS interpretation: tasks whose value functions require substantially different kernel characteristics cannot be simultaneously represented by a single fixed kernel. This perspective provides a principled explanation for the observed empirical behavior, including cases where cross-domain transfer fails.

These observations naturally motivate future research directions, such as enabling in-context adaptation of kernel parameters or learning kernel representations directly from context, thereby allowing the induced function space to adjust to task-specific structure.

Finally, our focus on policy evaluation enables a precise theoretical analysis but leaves open questions regarding policy improvement and control. Extending the kernel-based interpretation developed here to these settings, as well as to architectures with adaptive kernel structure, represents a natural direction for future work. We view the present work as a foundational step toward understanding how such capabilities might emerge in in-context reinforcement learning.

\begin{ack}
B.~He and L.~Lin were supported by the Duke University Center for AIDS Research (CFAR), an NIH-funded program (P30 AI064518-22), and by NIH grant P01 AI129859-09 for research on in-context learning, in-context reinforcement learning, and their applications to health, and other domains. We thank the Department of Computer Science at Duke University for computational resources. The content is solely the responsibility of the authors and does not necessarily represent the official views of the National Institutes of Health.
\end{ack}

\bibliographystyle{abbrvnat}
\bibliography{neurips_camera_ready}

@book{sutton1998reinforcement,
  title={Reinforcement Learning: An Introduction},
  author={Sutton, Richard S. and Barto, Andrew G.},
  isbn={9780262193986},
  lccn={97026416},
  series={A Bradford Book},
  address={Cambridge, MA},
  url={https://books.google.com/books?id=CAFR6IBF4xYC},
  year={1998},
  publisher={MIT Press}
}

@inproceedings{wang2025transformerslearntemporaldifference,
  title={Transformers Can Learn Temporal Difference Methods for In-Context Reinforcement Learning},
  author={Jiuqi Wang and Ethan Blaser and Hadi Daneshmand and Shangtong Zhang},
  booktitle={The Thirteenth International Conference on Learning Representations (ICLR)},
  year={2025},
  url={https://proceedings.iclr.cc/paper_files/paper/2025/hash/2e3435554b430bd8fe92a60c509929a0-Abstract-Conference.html},
  eprint={2405.13861},
  archivePrefix={arXiv}
}

@inproceedings{cheng2024transformersimplementfunctionalgradient,
  title={Transformers Implement Functional Gradient Descent to Learn Non-Linear Functions In Context},
  author={Xiang Cheng and Yuxin Chen and Suvrit Sra},
  booktitle={Proceedings of the 41st International Conference on Machine Learning (ICML)},
  series={Proceedings of Machine Learning Research},
  volume={235},
  pages={8002--8037},
  year={2024},
  publisher={PMLR},
  url={https://proceedings.mlr.press/v235/cheng24a.html},
  eprint={2312.06528},
  archivePrefix={arXiv}
}

@inproceedings{yu2021metaworldbenchmarkevaluationmultitask,
  title={{Meta-World}: A Benchmark and Evaluation for Multi-Task and Meta Reinforcement Learning},
  author={Yu, Tianhe and Quillen, Deirdre and He, Zhanpeng and Julian, Ryan and Hausman, Karol and Finn, Chelsea and Levine, Sergey},
  booktitle={Proceedings of the Conference on Robot Learning (CoRL 2019)},
  editor={Kaelbling, Leslie Pack and Kragic, Danica and Sugiura, Komei},
  series={Proceedings of Machine Learning Research},
  volume={100},
  pages={1094--1100},
  year={2020},
  publisher={PMLR},
  url={https://proceedings.mlr.press/v100/yu20a.html},
  eprint={1910.10897},
  archivePrefix={arXiv}
}

@inproceedings{xie2022explanationincontextlearningimplicit,
  title={An Explanation of In-context Learning as Implicit {Bayesian} Inference},
  author={Sang Michael Xie and Aditi Raghunathan and Percy Liang and Tengyu Ma},
  booktitle={The Tenth International Conference on Learning Representations (ICLR)},
  year={2022},
  url={https://openreview.net/forum?id=RdJVFCHjUMI},
  eprint={2111.02080},
  archivePrefix={arXiv}
}

@inproceedings{ICL:garg2023transformerslearnincontextcase,
  title={What Can Transformers Learn In-Context? {A} Case Study of Simple Function Classes},
  author={Shivam Garg and Dimitris Tsipras and Percy Liang and Gregory Valiant},
  booktitle={Advances in Neural Information Processing Systems},
  volume={35},
  pages={30583--30598},
  year={2022},
  url={https://papers.nips.cc/paper_files/paper/2022/hash/c529dba08a146ea8d6cf715ae8930cbe-Abstract-Conference.html},
  eprint={2208.01066},
  archivePrefix={arXiv}
}

@inproceedings{ICL:dai2023gptlearnincontextlanguage,
  title={Why Can {GPT} Learn In-Context? {L}anguage Models Secretly Perform Gradient Descent as Meta-Optimizers},
  author={Damai Dai and Yutao Sun and Li Dong and Yaru Hao and Shuming Ma and Zhifang Sui and Furu Wei},
  booktitle={Findings of the Association for Computational Linguistics: ACL 2023},
  editor={Rogers, Anna and Boyd-Graber, Jordan and Okazaki, Naoaki},
  pages={4005--4019},
  address={Toronto, Canada},
  publisher={Association for Computational Linguistics},
  month=jul,
  year={2023},
  doi={10.18653/v1/2023.findings-acl.247},
  url={https://aclanthology.org/2023.findings-acl.247},
  eprint={2212.10559},
  archivePrefix={arXiv}
}

@inproceedings{ICL:akyürek2023learningalgorithmincontextlearning,
  title={What learning algorithm is in-context learning? {I}nvestigations with linear models},
  author={Ekin Aky{\"u}rek and Dale Schuurmans and Jacob Andreas and Tengyu Ma and Denny Zhou},
  booktitle={The Eleventh International Conference on Learning Representations (ICLR)},
  year={2023},
  url={https://openreview.net/forum?id=0g0X4H8yN4I},
  eprint={2211.15661},
  archivePrefix={arXiv}
}

@inproceedings{ICL:wang2024largelanguagemodelslatent,
  title={Large Language Models Are Latent Variable Models: Explaining and Finding Good Demonstrations for In-Context Learning},
  author={Xinyi Wang and Wanrong Zhu and Michael Saxon and Mark Steyvers and William Yang Wang},
  booktitle={Advances in Neural Information Processing Systems},
  volume={36},
  pages={15614--15638},
  year={2023},
  url={https://proceedings.neurips.cc/paper_files/paper/2023/hash/3255a7554605a88800f4e120b3a929e1-Abstract-Conference.html},
  eprint={2301.11916},
  archivePrefix={arXiv}
}

@article{ICL:hahn2023theory,
  title={A theory of emergent in-context learning as implicit structure induction},
  author={Hahn, Michael and Goyal, Navin},
  journal={arXiv preprint arXiv:2303.07971},
  year={2023},
  url={https://arxiv.org/abs/2303.07971}
}

@inproceedings{vonoswald2023transformerslearnincontextgradient,
  title={Transformers Learn In-Context by Gradient Descent},
  author={von Oswald, Johannes and Niklasson, Eyvind and Randazzo, Ettore and Sacramento, Jo{\~a}o and Mordvintsev, Alexander and Zhmoginov, Andrey and Vladymyrov, Max},
  booktitle={Proceedings of the 40th International Conference on Machine Learning (ICML)},
  editor={Krause, Andreas and Brunskill, Emma and Cho, Kyunghyun and Engelhardt, Barbara and Sabato, Sivan and Scarlett, Jonathan},
  series={Proceedings of Machine Learning Research},
  volume={202},
  pages={35151--35174},
  year={2023},
  publisher={PMLR},
  url={https://proceedings.mlr.press/v202/von-oswald23a.html},
  eprint={2212.07677},
  archivePrefix={arXiv}
}

@inproceedings{ahn2023transformerslearnimplementpreconditioned,
  title={Transformers learn to implement preconditioned gradient descent for in-context learning},
  author={Kwangjun Ahn and Xiang Cheng and Hadi Daneshmand and Suvrit Sra},
  booktitle={Advances in Neural Information Processing Systems},
  volume={36},
  year={2023},
  url={https://proceedings.neurips.cc/paper_files/paper/2023/hash/8ed3d610ea4b68e7afb30ea7d01422c6-Abstract-Conference.html},
  eprint={2306.00297},
  archivePrefix={arXiv}
}

@article{zhang2023trainedtransformerslearnlinear,
  title={Trained Transformers Learn Linear Models In-Context},
  author={Ruiqi Zhang and Spencer Frei and Peter L. Bartlett},
  journal={Journal of Machine Learning Research},
  volume={25},
  number={49},
  pages={1--55},
  year={2024},
  url={https://jmlr.org/papers/v25/23-1042.html},
  eprint={2306.09927},
  archivePrefix={arXiv}
}

@inproceedings{mahankali2023stepgradientdescentprovably,
  title={One Step of Gradient Descent is Provably the Optimal In-Context Learner with One Layer of Linear Self-Attention},
  author={Arvind Mahankali and Tatsunori B. Hashimoto and Tengyu Ma},
  booktitle={The Twelfth International Conference on Learning Representations (ICLR)},
  year={2024},
  url={https://proceedings.iclr.cc/paper_files/paper/2024/hash/87f476af4053961667c2c08e9f4b850e-Abstract-Conference.html},
  eprint={2307.03576},
  archivePrefix={arXiv}
}

@inproceedings{tsai2019transformerdissectionunifiedunderstanding,
  title={Transformer Dissection: An Unified Understanding for Transformer's Attention via the Lens of Kernel},
  author={Yao-Hung Hubert Tsai and Shaojie Bai and Makoto Yamada and Louis-Philippe Morency and Ruslan Salakhutdinov},
  booktitle={Proceedings of the 2019 Conference on Empirical Methods in Natural Language Processing and the 9th International Joint Conference on Natural Language Processing (EMNLP-IJCNLP)},
  pages={4344--4353},
  address={Hong Kong, China},
  publisher={Association for Computational Linguistics},
  month=nov,
  year={2019},
  doi={10.18653/v1/D19-1443},
  url={https://aclanthology.org/D19-1443},
  eprint={1908.11775},
  archivePrefix={arXiv}
}

@inproceedings{choromanski2022rethinkingattentionperformers,
  title={Rethinking Attention with {Performers}},
  author={Krzysztof Choromanski and Valerii Likhosherstov and David Dohan and Xingyou Song and Andreea Gane and Tamas Sarlos and Peter Hawkins and Jared Davis and Afroz Mohiuddin and Lukasz Kaiser and David Belanger and Lucy Colwell and Adrian Weller},
  booktitle={The Ninth International Conference on Learning Representations (ICLR)},
  year={2021},
  url={https://openreview.net/forum?id=Ua6zuk0WRH},
  eprint={2009.14794},
  archivePrefix={arXiv}
}

@inproceedings{elnouby2021xcitcrosscovarianceimagetransformers,
  title={{XCiT}: Cross-Covariance Image Transformers},
  author={Alaaeldin El-Nouby and Hugo Touvron and Mathilde Caron and Piotr Bojanowski and Matthijs Douze and Armand Joulin and Ivan Laptev and Natalia Neverova and Gabriel Synnaeve and Jakob Verbeek and Herv{\'e} Jegou},
  booktitle={Advances in Neural Information Processing Systems},
  volume={34},
  year={2021},
  url={https://papers.nips.cc/paper/2021/hash/a655fbe4b8d7439994aa37ddad80de56-Abstract.html},
  eprint={2106.09681},
  archivePrefix={arXiv}
}

@inproceedings{nguyen2022improvingtransformersprobabilisticattention,
  title={Improving Transformers with Probabilistic Attention Keys},
  author={Tam Nguyen and Tan M. Nguyen and Dung D. Le and Duy Khuong Nguyen and Viet-Anh Tran and Richard G. Baraniuk and Nhat Ho and Stanley J. Osher},
  booktitle={Proceedings of the 39th International Conference on Machine Learning (ICML)},
  series={Proceedings of Machine Learning Research},
  volume={162},
  pages={16595--16621},
  year={2022},
  publisher={PMLR},
  url={https://proceedings.mlr.press/v162/nguyen22c.html},
  eprint={2110.08678},
  archivePrefix={arXiv}
}

@inproceedings{laskin2022incontextreinforcementlearningalgorithm,
  title={In-context Reinforcement Learning with Algorithm Distillation},
  author={Michael Laskin and Luyu Wang and Junhyuk Oh and Emilio Parisotto and Stephen Spencer and Richie Steigerwald and DJ Strouse and Steven Hansen and Angelos Filos and Ethan Brooks and Maxime Gazeau and Himanshu Sahni and Satinder Singh and Volodymyr Mnih},
  booktitle={The Eleventh International Conference on Learning Representations (ICLR)},
  year={2023},
  url={https://openreview.net/forum?id=hy0a5MMPUv},
  eprint={2210.14215},
  archivePrefix={arXiv}
}

@inproceedings{lee2023supervisedpretraininglearnincontext,
  title={Supervised Pretraining Can Learn In-Context Reinforcement Learning},
  author={Jonathan N. Lee and Annie Xie and Aldo Pacchiano and Yash Chandak and Chelsea Finn and Ofir Nachum and Emma Brunskill},
  booktitle={Advances in Neural Information Processing Systems},
  volume={36},
  year={2023},
  url={https://papers.neurips.cc/paper_files/paper/2023/hash/8644b61a9bc87bf7844750a015feb600-Abstract-Conference.html},
  eprint={2306.14892},
  archivePrefix={arXiv}
}

@inproceedings{kirsch2023towards,
  title={Towards General-Purpose In-Context Learning Agents},
  author={Louis Kirsch and James Harrison and C. Daniel Freeman and Jascha Sohl-Dickstein and J{\"u}rgen Schmidhuber},
  booktitle={NeurIPS 2023 Foundation Models for Decision Making Workshop},
  year={2023},
  url={https://openreview.net/forum?id=zDTqQVGgzH}
}

@inproceedings{shi2023crossepisodiccurriculumtransformeragents,
  title={Cross-Episodic Curriculum for Transformer Agents},
  author={Lucy Xiaoyang Shi and Yunfan Jiang and Jake Grigsby and Linxi Fan and Yuke Zhu},
  booktitle={Advances in Neural Information Processing Systems},
  volume={36},
  year={2023},
  url={https://proceedings.neurips.cc/paper_files/paper/2023/hash/001608167bb652337af5df0129aeaabd-Abstract.html},
  eprint={2310.08549},
  archivePrefix={arXiv}
}

@inproceedings{huang2024incontextdecisiontransformerreinforcement,
  title={In-Context Decision Transformer: Reinforcement Learning via Hierarchical Chain-of-Thought},
  author={Sili Huang and Jifeng Hu and Hechang Chen and Lichao Sun and Bo Yang},
  booktitle={Proceedings of the 41st International Conference on Machine Learning (ICML)},
  series={Proceedings of Machine Learning Research},
  volume={235},
  pages={19871--19885},
  year={2024},
  publisher={PMLR},
  url={https://proceedings.mlr.press/v235/huang24j.html},
  eprint={2405.20692},
  archivePrefix={arXiv}
}

@inproceedings{liu2023emergentagentictransformerchain,
  title={Emergent Agentic Transformer from Chain of Hindsight Experience},
  author={Hao Liu and Pieter Abbeel},
  booktitle={Proceedings of the 40th International Conference on Machine Learning (ICML)},
  series={Proceedings of Machine Learning Research},
  volume={202},
  pages={21362--21374},
  year={2023},
  publisher={PMLR},
  url={https://proceedings.mlr.press/v202/liu23a.html},
  eprint={2305.16554},
  archivePrefix={arXiv}
}

@inproceedings{lin2024transformersdecisionmakersprovable,
  title={Transformers as Decision Makers: Provable In-Context Reinforcement Learning via Supervised Pretraining},
  author={Licong Lin and Yu Bai and Song Mei},
  booktitle={The Twelfth International Conference on Learning Representations (ICLR)},
  year={2024},
  url={https://openreview.net/forum?id=CqMhoTVgHX},
  eprint={2310.08566},
  archivePrefix={arXiv}
}

@misc{moeini2025surveyincontextreinforcementlearning,
  title={A Survey of In-Context Reinforcement Learning},
  author={Amir Moeini and Jiuqi Wang and Jacob Beck and Ethan Blaser and Shimon Whiteson and Rohan Chandra and Shangtong Zhang},
  year={2025},
  eprint={2502.07978},
  archivePrefix={arXiv},
  primaryClass={cs.LG},
  url={https://arxiv.org/abs/2502.07978}
}

@article{Ormoneit2002,
  author  = {Ormoneit, Dirk and Sen, \'{S}aunak},
  title   = {Kernel-Based Reinforcement Learning},
  journal = {Machine Learning},
  volume  = {49},
  number  = {2},
  pages   = {161--178},
  year    = {2002},
  month   = nov,
  doi     = {10.1023/A:1017928328829},
  issn    = {1573-0565},
  url     = {https://doi.org/10.1023/A:1017928328829}
}

@inproceedings{engel2005reinforcementlearninggaussianprocesses,
  title     = {Reinforcement Learning with {Gaussian} Processes},
  author    = {Engel, Yaakov and Mannor, Shie and Meir, Ron},
  booktitle = {Proceedings of the 22nd International Conference on Machine Learning (ICML)},
  series    = {ICML '05},
  pages     = {201--208},
  year      = {2005},
  publisher = {Association for Computing Machinery},
  doi       = {10.1145/1102351.1102377},
  url       = {https://doi.org/10.1145/1102351.1102377}
}

@article{10.1109/TNN.2007.899161,
  author    = {Xu, Xin and Hu, Dewen and Lu, Xicheng},
  title     = {Kernel-Based Least Squares Policy Iteration for Reinforcement Learning},
  journal   = {IEEE Transactions on Neural Networks},
  volume    = {18},
  number    = {4},
  pages     = {973--992},
  month     = jul,
  year      = {2007},
  issn      = {1045-9227},
  publisher = {IEEE Press},
  doi       = {10.1109/TNN.2007.899161},
  url       = {https://doi.org/10.1109/TNN.2007.899161}
}

@inproceedings{10.1145/1553374.1553504,
  author    = {Taylor, Gavin and Parr, Ronald},
  title     = {Kernelized value function approximation for reinforcement learning},
  booktitle = {Proceedings of the 26th Annual International Conference on Machine Learning},
  series    = {ICML '09},
  pages     = {1017--1024},
  year      = {2009},
  isbn      = {9781605585161},
  publisher = {Association for Computing Machinery},
  address   = {New York, NY, USA},
  location  = {Montreal, Quebec, Canada},
  doi       = {10.1145/1553374.1553504},
  url       = {https://doi.org/10.1145/1553374.1553504}
}

@article{tarasov2026yes,
  title   = {Yes, {Q}-learning Helps Offline In-Context {RL}},
  author  = {Denis Tarasov and Alexander Nikulin and Ilya Zisman and Albina Klepach and Andrei Polubarov and Nikita Lyubaykin and Alexander Derevyagin and Igor Kiselev and Vladislav Kurenkov},
  journal = {Reinforcement Learning Journal},
  volume  = {7},
  year    = {2026},
  url     = {https://rlj.cs.umass.edu/2026/papers/Paper3.html},
  eprint  = {2502.17666},
  archivePrefix = {arXiv}
}

\appendix

\section{Construction Lemmas}
\label{app:construction}
\begin{lemma}[Construction for $\Attn^1$]
\label{lem:attn1}
Recall $\Attn^{\tilde{h}}_{K,Q,V}(Z) = VZ  M \tilde{h}(KZ,QZ)$, with
mask $M$ as defined in \eqref{e:M}.
Let
\begin{align*}
K &=
\begin{bmatrix}
\mI_d & \mathbf{0}_{d\times d} & \mathbf{0}_{d\times 1}\\
\mathbf{0}_{d\times d} & \mathbf{0}_{d\times d} & \mathbf{0}_{d\times 1}\\
\mathbf{0}_{1\times d} & \mathbf{0}_{1\times d} & 0
\end{bmatrix},\quad
Q =
\begin{bmatrix}
\mI_d & \mathbf{0}_{d\times d} & \mathbf{0}_{d\times 1}\\
\mathbf{0}_{d\times d} & \mathbf{0}_{d\times d} & \mathbf{0}_{d\times 1}\\
\mathbf{0}_{1\times d} & \mathbf{0}_{1\times d} & 0
\end{bmatrix},\\
V &=
\begin{bmatrix}
\mathbf{0}_{d\times d} & \mathbf{0}_{d\times d} & \mathbf{0}_{d\times 1}\\
\mathbf{0}_{d\times d} & \mathbf{0}_{d\times d} & \mathbf{0}_{d\times 1}\\
\mathbf{0}_{1\times d} & \mathbf{0}_{1\times d} & -\alpha_\ell
\end{bmatrix}.
\end{align*}
Assume $\tilde{h}$ induces the kernel $\kappa$ on the first $d$ coordinates, i.e.
$[\tilde{h}(KZ_\ell,QZ_\ell)]_{j,i} = \kappa(\ts{j},\ts{i})$.
Then for any $Z_\ell$ of the form \eqref{e:zl},
\begin{align*}
\Attn^{\tilde{h}}_{K,Q,V}(Z_\ell)
=
\begin{bmatrix}
0 & \cdots & 0 & 0 \\
0 & \cdots & 0 & 0 \\
\ty{0}_\ell & \cdots & \ty{n-1}_\ell & \ty{n}_\ell
\end{bmatrix},
\end{align*}
where for each $i\in\{0,1,\dots,n\}$,
\[
\ty{i}_\ell \;:=\; -\alpha_\ell \sum_{j=0}^{n-1} \tb{j}_\ell \, \kappa(\ts{j},\ts{i}).
\]
\end{lemma}

\begin{proof}
Write $Z_\ell=\begin{bmatrix}Z_\ell^{(1)}\\ Z_\ell^{(2)}\\ Z_\ell^{(3)}\end{bmatrix}$ where
$Z_\ell^{(1)}\in\Re^{d\times(n+1)}$ is the first block-row, $Z_\ell^{(2)}\in\Re^{d\times(n+1)}$ is the second
block-row, and $Z_\ell^{(3)}\in\Re^{1\times(n+1)}$ is the last row.

\paragraph{Step 1: $KZ_\ell$ and $QZ_\ell$.}
By the definitions of $K$ and $Q$, we have
\[
KZ_\ell =
\begin{bmatrix}
Z_\ell^{(1)}\\ \mathbf{0}_{d\times(n+1)}\\ \mathbf{0}_{1\times(n+1)}
\end{bmatrix},
\qquad
QZ_\ell =
\begin{bmatrix}
Z_\ell^{(1)}\\ \mathbf{0}_{d\times(n+1)}\\ \mathbf{0}_{1\times(n+1)}
\end{bmatrix}.
\]
Hence, by the assumption on $\tilde{h}$ in \eqref{e:th} and \eqref{e:kappa},
\[
\lrb{\tilde{h}(KZ_\ell,QZ_\ell)}_{j,i} = \kappa(\ts{j},\ts{i}).
\]

\paragraph{Step 2: $VZ_\ell$.}
By the definition of $V$, $VZ_\ell$ is zero in its first $2d$ rows and equals $-\alpha_\ell Z_\ell^{(3)}$
in its last row:
\[
VZ_\ell =
\begin{bmatrix}
\mathbf{0}_{d\times(n+1)}\\
\mathbf{0}_{d\times(n+1)}\\
-\alpha_\ell Z_\ell^{(3)}
\end{bmatrix}.
\]

\paragraph{Step 3: masking by $M$.}
Since $M=\begin{bmatrix}\mI_n & 0\\ 0 & 0\end{bmatrix}$, multiplying on the right zeros out the last (query) column,
so
\[
VZ_\ell\, M =
\begin{bmatrix}
\mathbf{0}_{d\times(n+1)}\\
\mathbf{0}_{d\times(n+1)}\\
-\alpha_\ell\,[\tb{0}_\ell,\dots,\tb{n-1}_\ell,0]
\end{bmatrix}.
\]

\paragraph{Step 4: multiply by affinity matrix.}
Therefore
\[
\Attn^{\tilde{h}}_{K,Q,V}(Z_\ell) = (VZ_\ell\,M)\,H
\]
has zero in its first $2d$ rows, and for each $i\in\{0,\dots,n\}$ its last-row entry is
\[
-\alpha_\ell \sum_{j=0}^{n-1} \tb{j}_\ell \, H_{j,i}
= -\alpha_\ell \sum_{j=0}^{n-1} \tb{j}_\ell \, \kappa(\ts{j},\ts{i})
\;=\; \ty{i}_\ell.
\]
This concludes the proof.
\end{proof}

\begin{lemma}[Construction for $\Attn^2$]
\label{lem:attn2}
Recall $\Attn^{\tilde{h}}_{K,Q,V}(Z) = VZ\, M\, \tilde{h}(KZ,QZ)$, with mask $M$ as defined in \eqref{e:M}.
Let
\begin{align*}
K &=
\begin{bmatrix}
\mI_d & \mathbf{0}_{d\times d} & \mathbf{0}_{d\times 1}\\
\mathbf{0}_{d\times d} & \mathbf{0}_{d\times d} & \mathbf{0}_{d\times 1}\\
\mathbf{0}_{1\times d} & \mathbf{0}_{1\times d} & 0
\end{bmatrix},\quad
Q =
\begin{bmatrix}
\mathbf{0}_{d\times d} & \mI_d & \mathbf{0}_{d\times 1}\\
\mathbf{0}_{d\times d} & \mathbf{0}_{d\times d} & \mathbf{0}_{d\times 1}\\
\mathbf{0}_{1\times d} & \mathbf{0}_{1\times d} & 0
\end{bmatrix},\\
V &=
\begin{bmatrix}
\mathbf{0}_{d\times d} & \mathbf{0}_{d\times d} & \mathbf{0}_{d\times 1}\\
\mathbf{0}_{d\times d} & \mathbf{0}_{d\times d} & \mathbf{0}_{d\times 1}\\
\mathbf{0}_{1\times d} & \mathbf{0}_{1\times d} & \gamma \alpha_\ell
\end{bmatrix}.
\end{align*}
Assume $\tilde{h}$ induces the kernel $\kappa$ on the first $d$ coordinates in the sense that,
for $Z_\ell$ of the form \eqref{e:zl},
\[
\lrb{\tilde{h}(KZ_\ell,QZ_\ell)}_{j,i}
=
\begin{cases}
\kappa(\ts{j},\ts{i+1}), & i\in\{0,1,\dots,n-1\},\\
\kappa(\ts{j},\ts{\pad}), & i=n.
\end{cases}
\]
Then for any $Z_\ell$ of the form \eqref{e:zl},
\begin{align*}
\Attn^{\tilde{h}}_{K,Q,V}(Z_\ell)
=
\begin{bmatrix}
0 & \cdots & 0 & 0 \\
0 & \cdots & 0 & 0 \\
\ty{0}_\ell & \cdots & \ty{n-1}_\ell & \ty{n}_\ell
\end{bmatrix},
\end{align*}
where
\[
\ty{i}_\ell \;:=\;
\begin{cases}
\gamma \alpha_\ell \sum_{j=0}^{n-1} \tb{j}_\ell \, \kappa(\ts{j},\ts{i+1}), & i\in\{0,1,\dots,n-1\},\\[4pt]
\gamma \alpha_\ell \sum_{j=0}^{n-1} \tb{j}_\ell \, \kappa(\ts{j},\ts{\pad}), & i=n.
\end{cases}
\]
\end{lemma}

\begin{proof}
Write $Z_\ell=\begin{bmatrix}Z_\ell^{(1)}\\ Z_\ell^{(2)}\\ Z_\ell^{(3)}\end{bmatrix}$ where
$Z_\ell^{(1)}\in\Re^{d\times(n+1)}$ is the first block-row, $Z_\ell^{(2)}\in\Re^{d\times(n+1)}$ is the second
block-row, and $Z_\ell^{(3)}\in\Re^{1\times(n+1)}$ is the last row.

\paragraph{Step 1: $KZ_\ell$ and $QZ_\ell$.}
By the definition of $K$, we have
\[
KZ_\ell =
\begin{bmatrix}
Z_\ell^{(1)}\\ \mathbf{0}_{d\times(n+1)}\\ \mathbf{0}_{1\times(n+1)}
\end{bmatrix}.
\]
By the definition of $Q$, we have
\[
QZ_\ell =
\begin{bmatrix}
Z_\ell^{(2)}\\ \mathbf{0}_{d\times(n+1)}\\ \mathbf{0}_{1\times(n+1)}
\end{bmatrix}.
\]
Hence, by the assumed kernel-inducing property of $\tilde{h}$ (stated in the lemma),
\[
\lrb{\tilde{h}(KZ_\ell,QZ_\ell)}_{j,i}=
\begin{cases}
\kappa(\ts{j},\ts{i+1}), & i\in\{0,\dots,n-1\},\\
\kappa(\ts{j},\ts{\pad}), & i=n.
\end{cases}
\]

\paragraph{Step 2: $VZ_\ell$.}
By the definition of $V$, $VZ_\ell$ is zero in its first $2d$ rows and equals $\gamma\alpha_\ell Z_\ell^{(3)}$
in its last row:
\[
VZ_\ell =
\begin{bmatrix}
\mathbf{0}_{d\times(n+1)}\\
\mathbf{0}_{d\times(n+1)}\\
\gamma\alpha_\ell Z_\ell^{(3)}
\end{bmatrix}.
\]

\paragraph{Step 3: masking by $M$.}
Multiplying on the right by $M$ zeros out the last (query) column, so
\[
VZ_\ell\, M =
\begin{bmatrix}
\mathbf{0}_{d\times(n+1)}\\
\mathbf{0}_{d\times(n+1)}\\
\gamma\alpha_\ell\,[\tb{0}_\ell,\dots,\tb{n-1}_\ell,0]
\end{bmatrix}.
\]

\paragraph{Step 4: multiply by affinity matrix.}
Therefore
\[
\Attn^{\tilde{h}}_{K,Q,V}(Z_\ell) = (VZ_\ell\,M)\,\tilde{h}(KZ_\ell,QZ_\ell)
\]
has zero in its first $2d$ rows, and for each $i\in\{0,\dots,n\}$ its last-row entry is
\[
\gamma\alpha_\ell \sum_{j=0}^{n-1} \tb{j}_\ell\, \lrb{\tilde{h}(KZ_\ell,QZ_\ell)}_{j,i}.
\]
Substituting the expression for $H_{j,i}$ yields
\begin{align*}
&\bigl[\Attn^{\tilde{h}}_{K,Q,V}(Z_\ell)\bigr]_{2d+1,i}\\
&=
\begin{cases}
\gamma \alpha_\ell \sum_{j=0}^{n-1} \tb{j}_\ell\, \kappa(\ts{j},\ts{i+1}), & i\in\{0,\dots,n-1\},\\[4pt]
\gamma \alpha_\ell \sum_{j=0}^{n-1} \tb{j}_\ell\, \kappa(\ts{j},\ts{\pad}), & i=n,
\end{cases}
\end{align*}
which is exactly $\ty{i}_\ell$ as defined in the lemma, completing the proof.
\end{proof}

\newpage
\section{Learning Curves for Different Domains}
\label{app:learning_curve}
To provide additional evidence supporting our results, we include in this section the learning curves for each domain, averaged over five random seeds. For comparison, we also report the learning curves of the linear transformer baseline from~\cite{wang2025transformerslearntemporaldifference}.

Specifically, the learning curves report the TD loss values used to optimize the model parameter $\alpha$, as described in Section~\ref{sec:implementation}. They illustrate how the adapted learning rate in the induced kernel regression affects the convergence of the TD loss. Figure~\ref{fig:loss_learning_curves} present the results. We use identical training configurations for both the non-linear and linear transformers. The only difference is that the linear transformer uses the default learning rate of $0.001$, as reported by the authors. The failure of the linear transformer to converge on MetaWorld tasks in terms of the TD error is expected, given that it is restricted to a linear function class, which is inadequate for the underlying value functions.
\begin{figure*}[hb]
    \centering

    \begin{subfigure}{0.40\textwidth}
        \centering
        \includegraphics[width=0.95\linewidth]{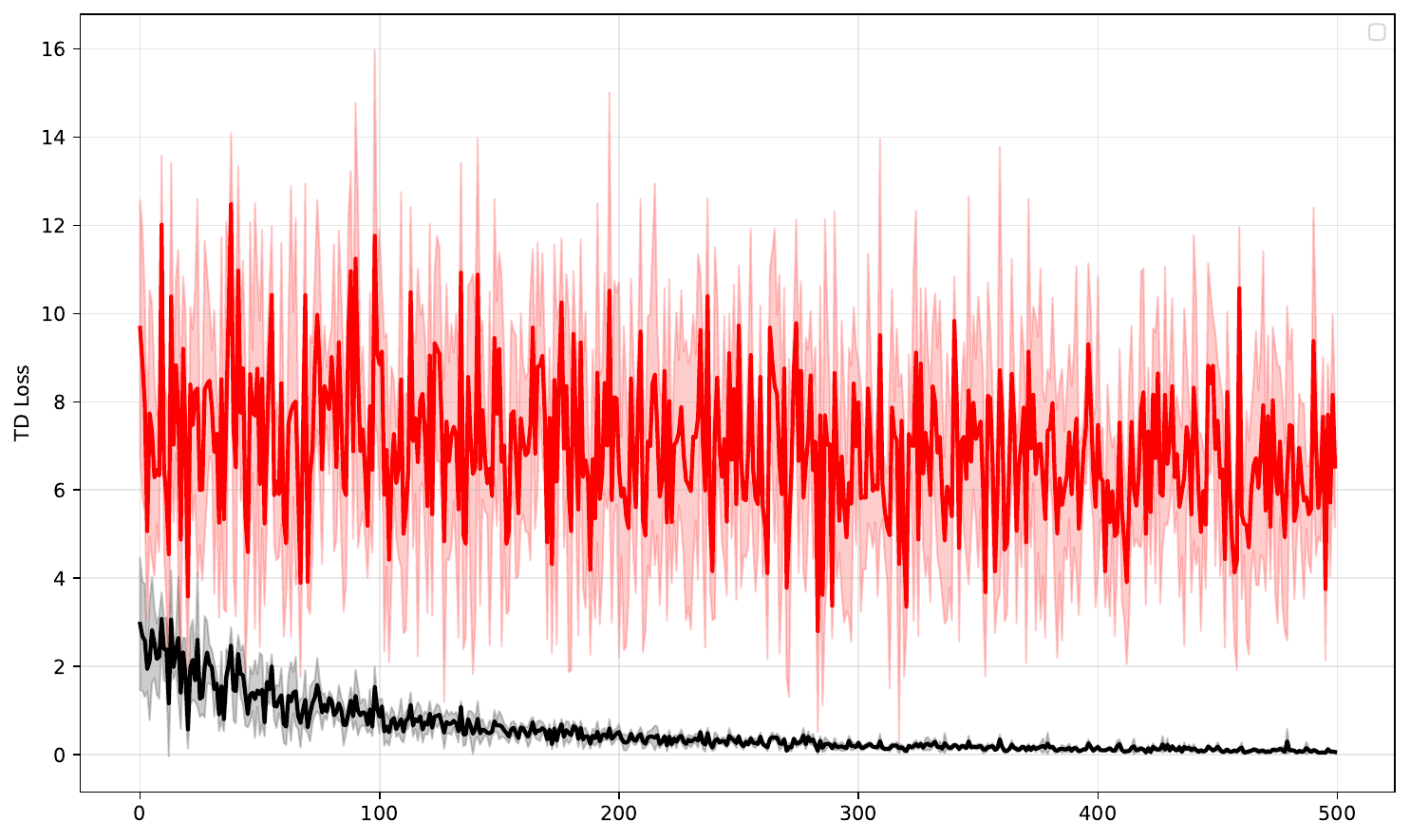}
        \caption{Pick-Place-v3}
        \label{fig:learning_curve:pick_and_place_learning_curve}
    \end{subfigure}
    \hfill
    \begin{subfigure}{0.40\textwidth}
        \centering
        \includegraphics[width=0.95\linewidth]{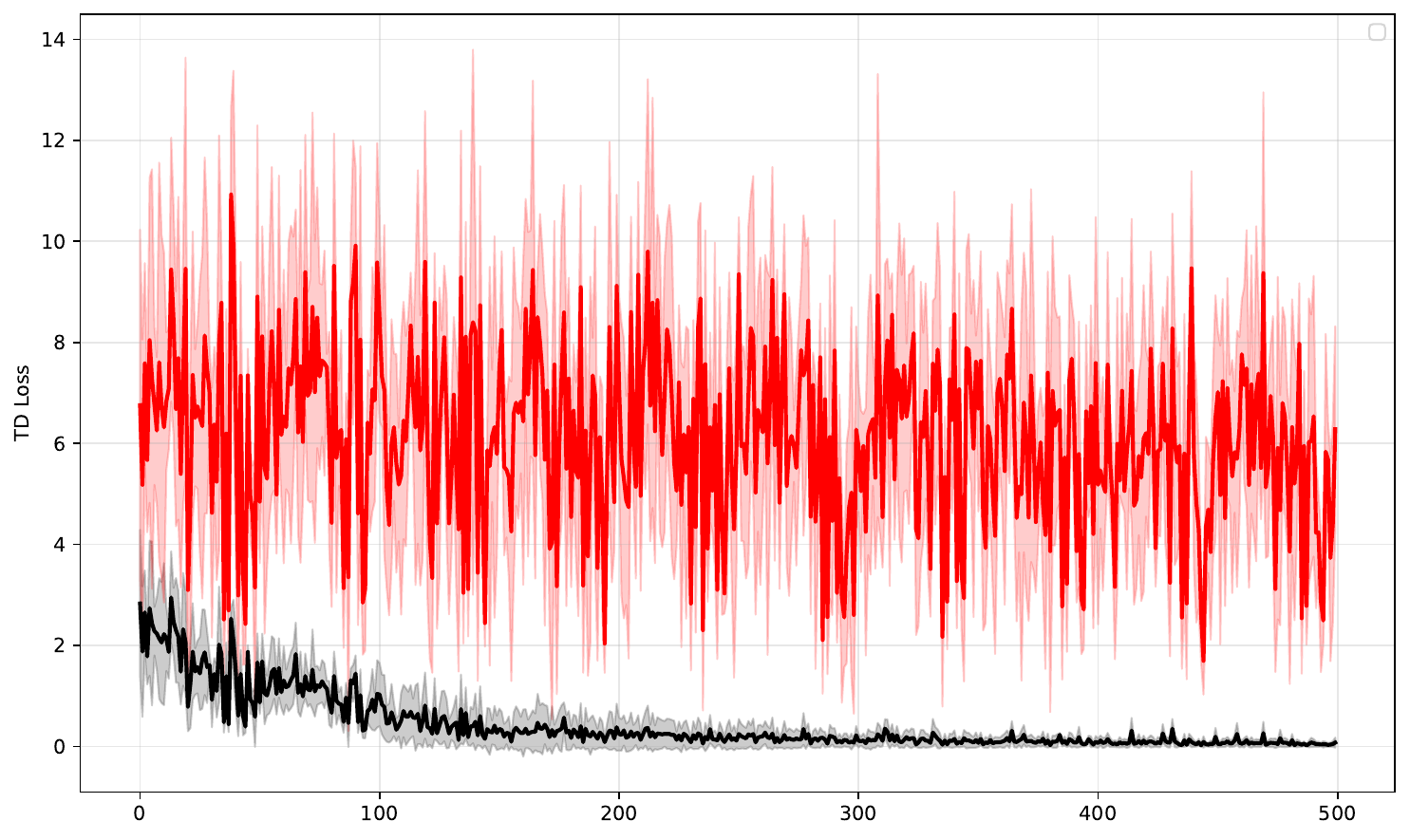}
        \caption{Pick-Place-Wall-v3}
        \label{fig:learning_curve:pick_and_place_wall_learning_curve}
    \end{subfigure}

    \vspace{0.25em}

    \begin{subfigure}{0.40\textwidth}
        \centering
        \includegraphics[width=0.95\linewidth]{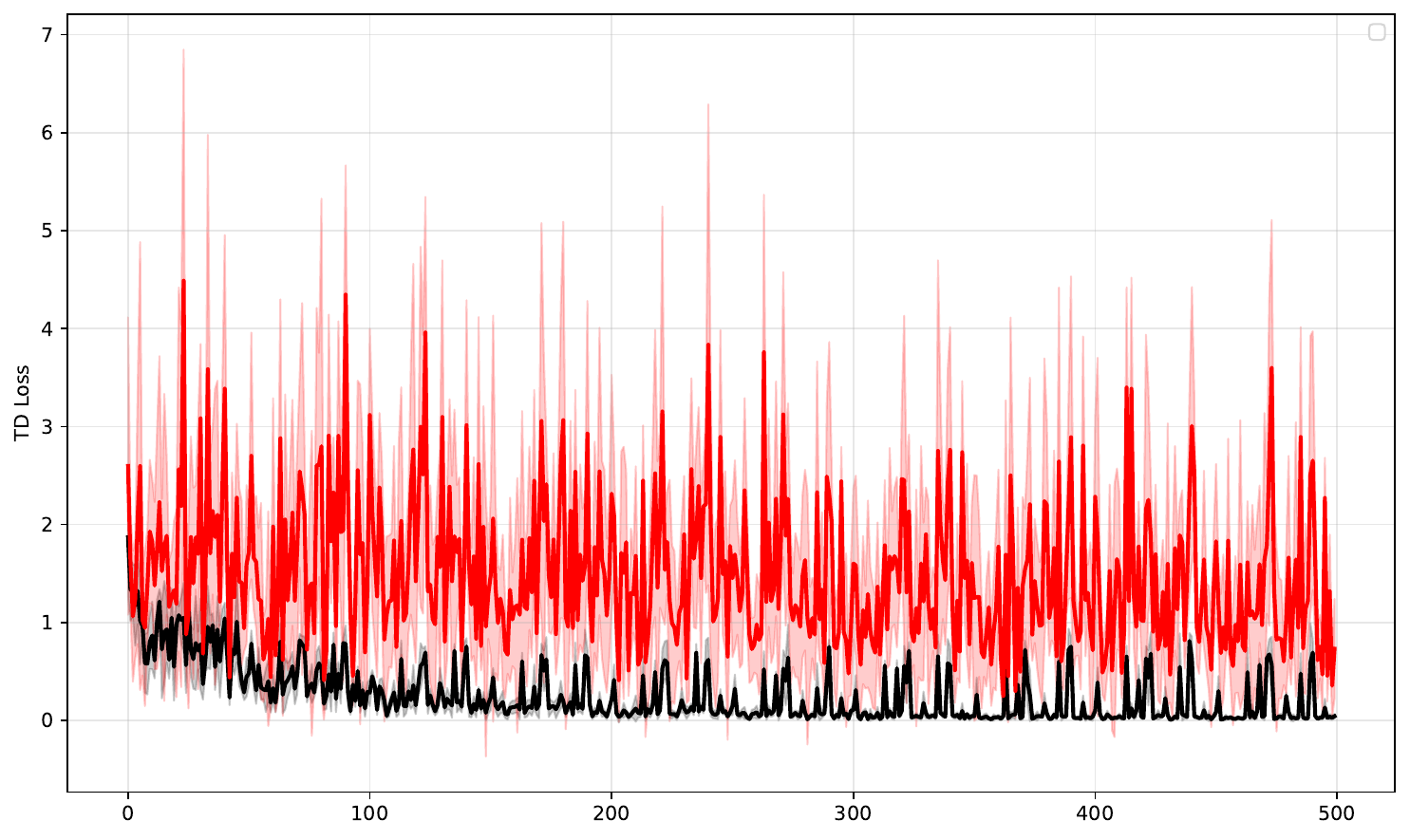}
        \caption{Shelf-Place-v3}
        \label{fig:learning_curve:shelf_place_learning_curve}
    \end{subfigure}
    \hfill
    \begin{subfigure}{0.40\textwidth}
        \centering
        \includegraphics[width=0.95\linewidth]{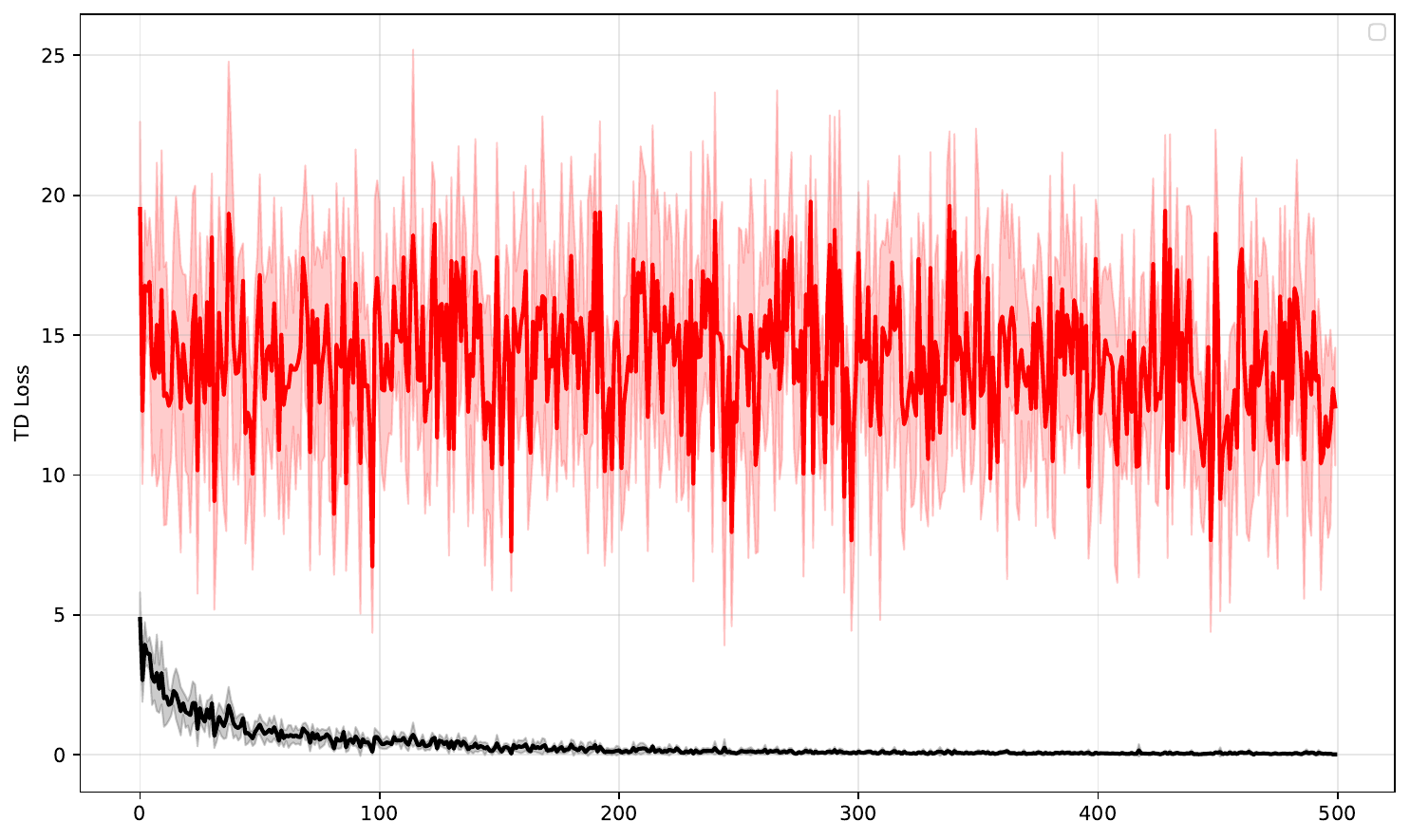}
        \caption{Plate-Slide-v3}
        \label{fig:learning_curve:Plate-Slide}
    \end{subfigure}

    \vspace{0.25em}

    \begin{subfigure}{0.45\textwidth}
        \centering
        \includegraphics[width=0.90\linewidth]{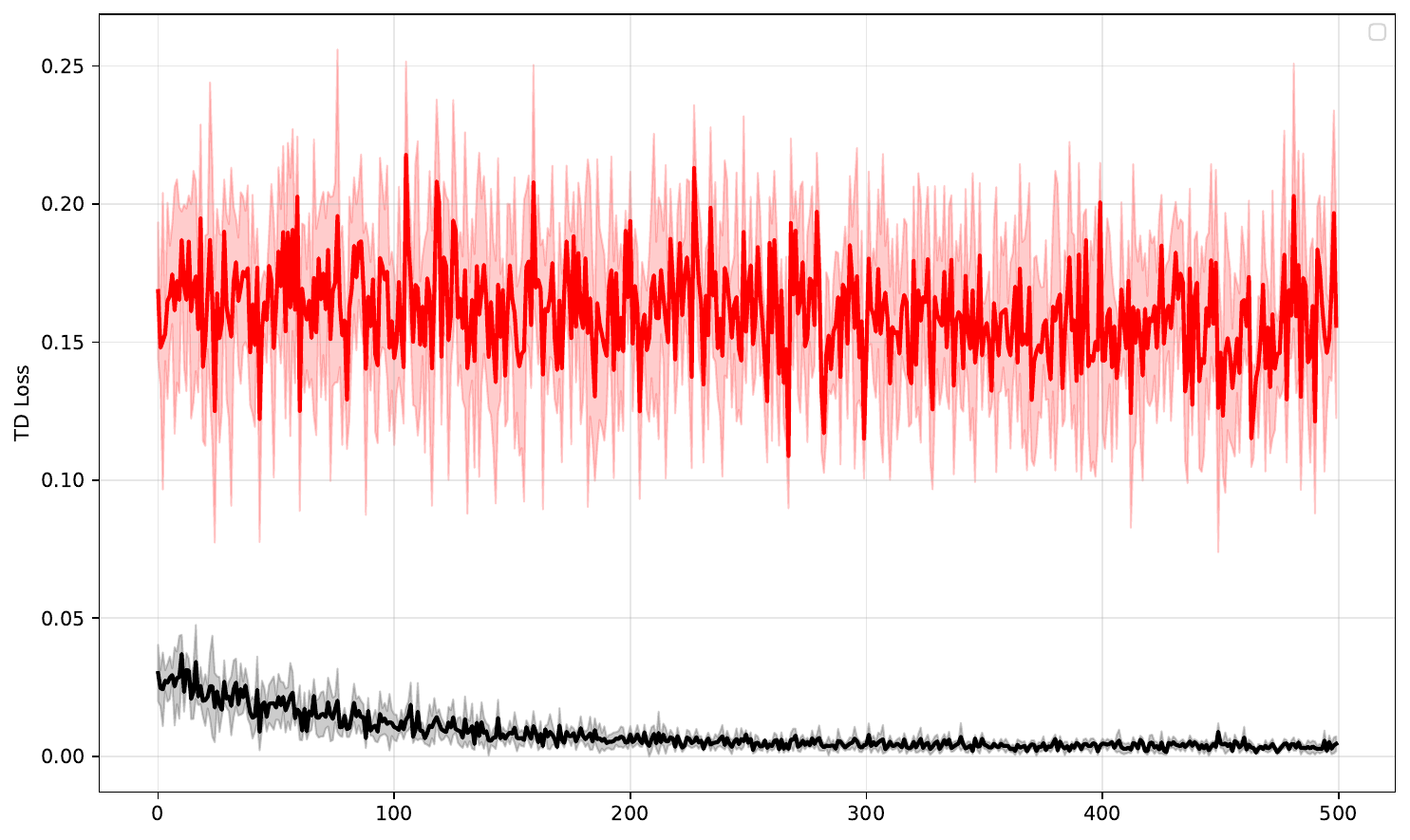}
        \caption{Button-Press-v3}
        \label{fig:learning_curve:Button-Press}
    \end{subfigure}
    \caption*{\footnotesize
        \textcolor{red}{\rule{1.5em}{2pt}} Linear Transformer \quad
        \textcolor{black}{\rule{1.5em}{2pt}} Our Non-Linear Transformer
    }
    \caption{Learning curves across five MetaWorld domains. Each curve is averaged over five random seeds.}
    \label{fig:loss_learning_curves}
\end{figure*}

\section{One Model for All Domains}
\label{app:one_for_all}
It is also natural to ask whether a model trained on a single domain can be reused across multiple domains, a question that directly reflects the title of this work. As shown in Appendix~\ref{app:hyperparameters}, four domains share a common kernel temperature, whereas \textbf{Button-Press-v3} requires a distinct setting. This observation suggests that models may generalize effectively across the four domains but fail to transfer to \textbf{Button-Press-v3}, a discrepancy whose underlying causes are discussed in Section~\ref{sec:limitations}.

To test this hypothesis, we report a cross-domain evaluation matrix averaged over five random seeds, where each row corresponds to a training domain and each column corresponds to an evaluation domain. For example, the first row shows the performance of models trained on \textbf{Pick-Place-v3} and evaluated across all five domains using checkpoints saved throughout training. Thus, the subplots along the diagonal correspond to settings where the training and evaluation domains are the same. Importantly, these models are evaluated directly in the target domains without any additional fine-tuning or training. Figure~\ref{fig:one_for_all} presents the resulting TD loss curves, showing that the TD error continues to decrease across the four domains, while diverging on Button-Press-v3, consistent with our hypothesis.

Specifically, for models trained on \textbf{Pick-Place-v3}, \textbf{Pick-Place-Wall-v3}, \textbf{Shelf-Place-v3}, and \textbf{Plate-Slide-v3}, corresponding to the first four rows of the matrix, the TD loss generally continues to decrease when evaluated on the first four domains, corresponding to the first four columns of the matrix. The main exception is that models trained on \textbf{Pick-Place-v3} exhibit divergence on \textbf{Plate-Slide-v3} toward the end of training. This behavior may arise because \textbf{Plate-Slide-v3} exhibits a substantially different value range, and excessive adaptation of the scaling parameter $\alpha$ and the temperature parameter can induce an RKHS that is no longer well aligned with this domain. Nevertheless, the model converges over a large portion of the training trajectory. By contrast, the divergence of models evaluated on \textbf{Button-Press-v3} is expected, as discussed above. For models trained on \textbf{Button-Press-v3}, corresponding to the last row of the matrix, we observe that the value ranges differ substantially from those of the first four training domains and that the TD loss exhibits significant divergence when transferred to other domains. This further suggests that \textbf{Button-Press-v3} is substantially different from the other domains, consistent with our hypothesis.
\begin{figure}[t]
    \centering
    \includegraphics[width=0.98\columnwidth]{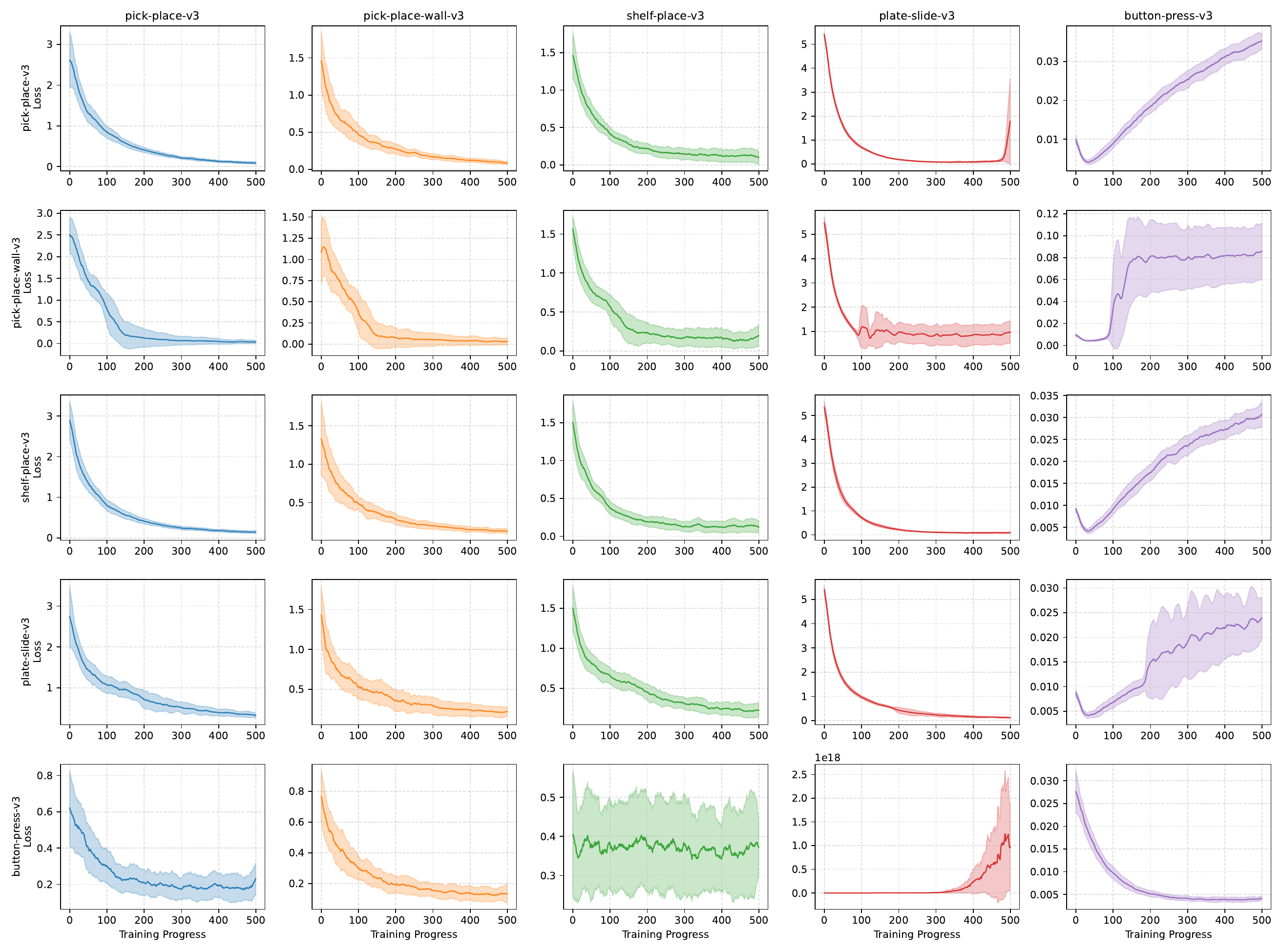}
    \caption{Models trained on a single domain are applied to all domains simultaneously.}
    \label{fig:one_for_all}
\end{figure}

\section{Discontinuous State Value Function and TD Error}
\label{app:discontinuous_state_value}
A key assumption of the model is that the underlying state-value function lies within the induced RKHS, under which the function can be accurately approximated via in-context learning. When this assumption is violated, the model's performance degrades. Under the original MetaWorld reward settings, the induced state-value function is highly discontinuous, making it difficult for the TD error to converge to zero across all states.

To illustrate this point, we take a fully trained model on Pick-Place-v3 and visualize, along an entire episode, the TD error at each state against the corresponding discounted return. Figure~\ref{fig:value_discontinuity_TD_Error} illustrates the result. Note that we report the \itbf{square root} of the TD loss to place it on a scale comparable to the state value. After this rescaling, the TD error remains on the order of \itbf{10}, which is still substantially smaller than the state-value scale of approximately \itbf{100}. The TD error remains near zero during the initial time steps but increases as the time horizon grows. Overall, \btxt{this behavior stems from properties of the environment rather than a fundamental mistake of the model}. Simple modifications such as reward rescaling---which effectively rescales the value-function landscape---can help mitigate this issue. In light of this behavior, we report TD loss values over the initial time steps.
\begin{figure}[t]
    \centering
    \includegraphics[width=0.5\columnwidth]{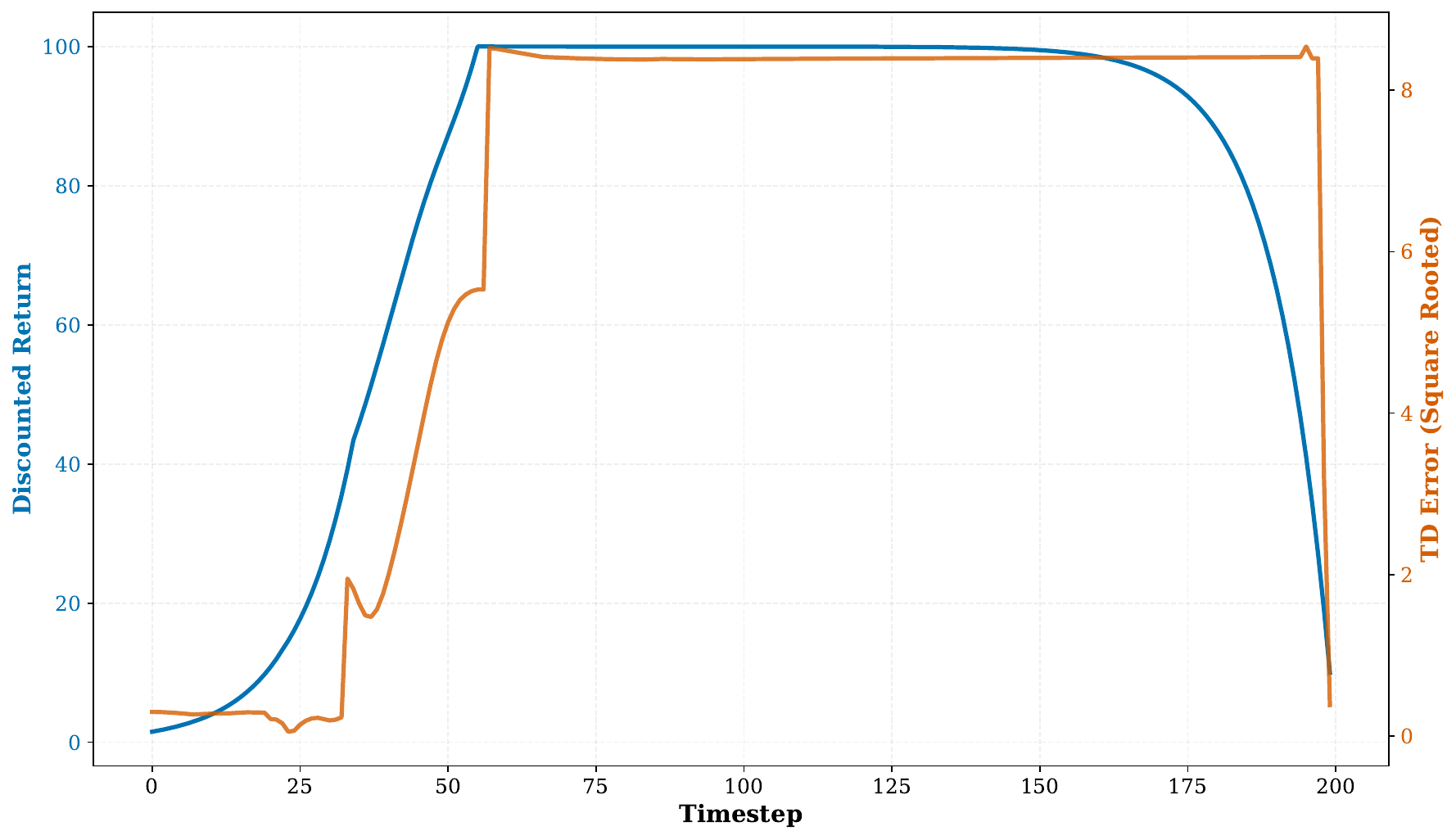}
    \caption{Illustration of the induced state-value function under the original MetaWorld reward settings. The value function exhibits large discontinuities, violating the smoothness assumptions implicit in the RKHS approximation.}
    \label{fig:value_discontinuity_TD_Error}
\end{figure}

\section{Details for Experiments}
\subsection{Environment Setting}
\label{app:noise_level}
To construct the MRPs from the MetaWorld environments, as discussed in Section~\ref{sec:experiments}, we inject Gaussian noise at varying levels into the actions generated by the optimal policies. For completeness, we report the corresponding noise standard deviations and the estimated initial state values for each domain. All state values are computed using a discount factor of $\gamma = 0.9$. Table~\ref{tab:mrp_noise} presents the results. Notably, the relationships among the initial state values $V(s_0)$ align closely with the trends observed in Figure~\ref{fig:TD_over_tasks}. For example, in Button-Press-v3, the $V(s_0)$ values are relatively close to one another, which leads to greater overlap among the corresponding curves in Figure~\ref{fig:TD_over_tasks}.

The values span a broader range, reflecting our choice of domains designed to evaluate the model under diverse settings.

\subsection{Hyperparameter Settings}
\label{app:hyperparameters}
We provide the hyperparameters used across all experiments in this work in table~\ref{tab:hyperparams}. While the domains share a set of common hyperparameters, each domain also includes domain-specific settings. In particular, as discussed in Section~\ref{sec:limitations}, the current construction requires different activation temperatures (i.e., kernel parameters) across domains. While our experiments show that a single set of kernel parameters can perform well across multiple domains, we nevertheless report common and domain-specific hyperparameters separately to emphasize that these parameters may still require adjustment. We note that these hyperparameters could be further optimized to yield stronger empirical performance.

\begin{table}[H]
    \centering
    \caption{Gaussian noise levels injected into optimal-policy actions and the corresponding estimated initial state values for each MetaWorld domain.}
    \label{tab:mrp_noise}
    \begin{tabular}{lcc}
        \toprule
        Domain & Noise Std. ($\sigma$) & Initial State Value $V(s_0)$ \\
        \midrule
        {Pick-Place-v3}
            & {0.0 (low)} & 1.52 \\
            & {0.3 (medium)} & 1.31 \\
            & {0.5 (high)} & 0.79 \\
        \midrule
        {Pick-Place-Wall-v3}
            & 0.0 & 0.91 \\
            & 0.2 & 0.54 \\
            & 0.4 & 0.16 \\
        \midrule
        {Shelf-Place-v3}
            & 0.0 & 2.06 \\
            & 0.4 & 1.98 \\
            & 0.8 & 1.06 \\
        \midrule
        {Plate-Slide-v3}
            & 0.0 & 10.33 \\
            & 0.5 & 9.89 \\
            & 0.9 & 9.12 \\
        \midrule
        {Button-Press-v3}
            & 0.0 & 3.01 \\
            & 0.6 & 2.94 \\
            & 0.9 & 2.87 \\
        \bottomrule
    \end{tabular}
\end{table}

\begin{table}[H]
\centering
\small
\setlength{\tabcolsep}{6pt}
\renewcommand{\arraystretch}{1.15}
\caption{Hyperparameters. We report (i) common settings shared across all domains and (ii) domain-specific overrides.}
\label{tab:hyperparams}
\begin{tabular}{ll}
\toprule
\multicolumn{2}{l}{\textbf{Common (all domains)}} \\
\midrule
Batch size & 32 \\
Optimizer & Adam \\
Weight decay & $1\times10^{-6}$ \\
Context length & 1000 \\
State vector dimensionality & 39 \\
Number of transformer layers & 30 \\
Activation function & Exponential \\
Discount factor $\gamma$ & 0.9 \\
\midrule
\multicolumn{2}{l}{\textbf{Pick-Place-v3}} \\
\midrule
Activation temperature & 10.0 \\
Learning rate & 0.01 \\
\midrule
\multicolumn{2}{l}{\textbf{Pick-Place-Wall-v3}} \\
\midrule
Activation temperature & 10.0 \\
Learning rate & 0.01 \\
\midrule
\multicolumn{2}{l}{\textbf{Shelf-Place-v3}} \\
\midrule
Activation temperature & 10.0 \\
Learning rate & 0.01 \\
\midrule
\multicolumn{2}{l}{\textbf{Plate-Slide-v3}} \\
\midrule
Activation temperature & 10.0 \\
Learning rate & 0.01 \\
\midrule
\multicolumn{2}{l}{\textbf{Button-Press-v3}} \\
\midrule
Activation temperature & 2.5 \\
Learning rate & 0.0001 \\
\bottomrule
\end{tabular}
\end{table}

\section{Synthetic Environment}
\label{app:synthetic_env}
\subsection{Additional Illustrations}
We provide additional illustrations to examine the empirical results on the synthetic environment from multiple perspectives. Figure~\ref{appfig:across_seeds} shows the model-approximated state value functions under different random seeds. For each seed, the context is sampled independently. Across these different contexts, the transformer consistently recovers the shape of the state value function with high accuracy.

We further conduct ablation studies on context length and model depth. Figure~\ref{appfig:across_context} varies the number of transitions provided to the transformer while fixing the number of layers at $32$, whereas Figure~\ref{appfig:across_layers} varies the number of transformer layers while fixing the number of context transitions at $32$. In both studies, we sweep over ${2,4,8,16,32}$. These results show how the quality of the value-function approximation changes as the amount of context and the number of forward-pass TD update steps increase, consistent with our expectation that sufficient context and model depth are necessary for the model to achieve satisfactory performance.

\begin{figure}[t]
    \centering
    \begin{subfigure}{1.0\columnwidth}
        \centering
        \includegraphics[width=\linewidth]{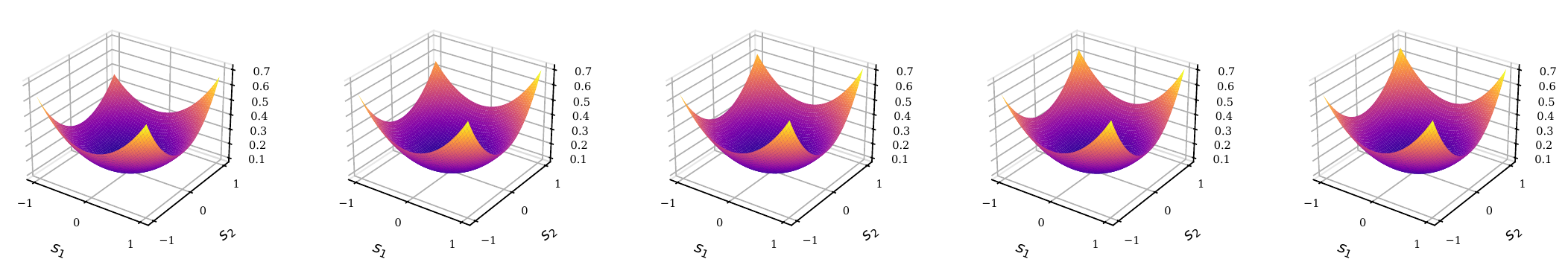}
        \caption{Model-approximated state value functions across random seeds.}
        \label{appfig:across_seeds}
    \end{subfigure}
    \hfill
    \begin{subfigure}{1.0\columnwidth}
        \centering
        \includegraphics[width=\linewidth]{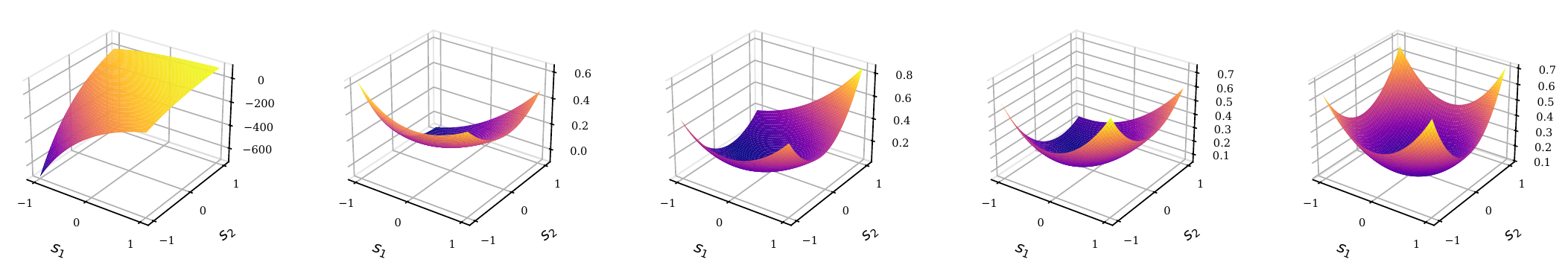}
        \caption{Model-approximated state value functions as the number of context transitions varies. From left to right, the number of context transitions is $2$, $4$, $8$, $16$, and $32$.}
        \label{appfig:across_context}
    \end{subfigure}
    \hfil
    \begin{subfigure}{1.0\columnwidth}
        \centering
        \includegraphics[width=\linewidth]{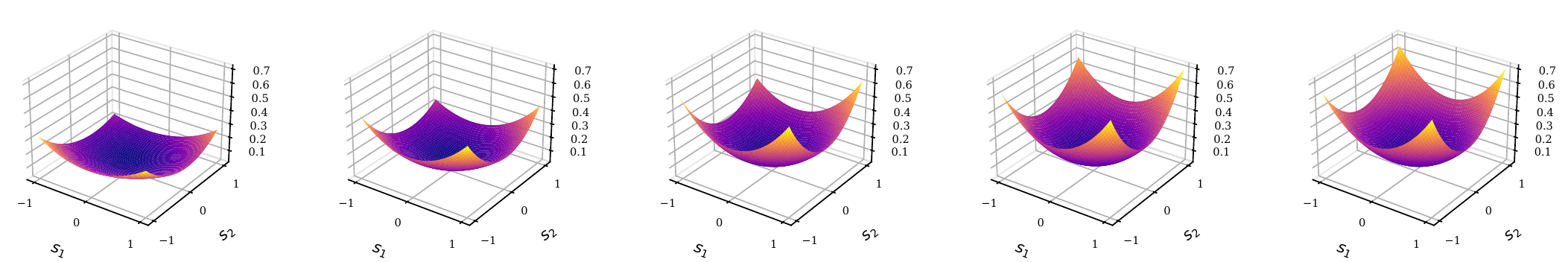}
        \caption{Model-approximated state value functions as the number of transformer layers varies. From left to right, the number of context transitions is $2$, $4$, $8$, $16$, and $32$.}
        \label{appfig:across_layers}
    \end{subfigure}
    \hfil
    \begin{subfigure}{0.2\columnwidth}
        \centering
        \includegraphics[width=\linewidth]{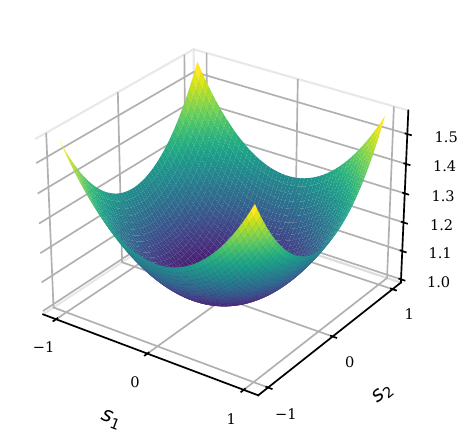}
        \caption{Ground-truth state value function.}
        \label{appfig:ground_truth}
    \end{subfigure}
    \caption{Comparison between the model-approximated and ground-truth state value functions.}
    \label{fig:app_syntheic_env}
\end{figure}

\subsection{Environment Details}
We provide details on the synthetic MRP used in our experiments. The purpose of this environment is to construct a controlled setting in which the ground-truth state value function lies in an RKHS and can therefore be theoretically approximated via kernel regression. If the transformer can recover the underlying state value function from environment transitions provided as context, this empirically supports our claim that the transformer implements kernel TD regression in its forward pass.

\paragraph{State space and transition dynamics.}
We define a Markov Reward Process (MRP) as a tuple
\[
    \mathcal{M} = (\mathcal{S}, P, r, \gamma),
\]
where the state space is the two-dimensional compact domain
\[
    \mathcal{S} = [-1,1]^2 \subset \mathbb{R}^2.
\]
Given a current state \(s_t \in \mathcal{S}\), the next state is generated according to the transition kernel \(P(\cdot \mid s_t)\) induced by the linear Gaussian dynamics
\begin{equation}
    s_{t+1} = \rho s_t + \epsilon_t,
    \qquad
    \epsilon_t \sim \mathcal{N}(0,\sigma^2 I_2),
\label{eq:synthetic_transition}
\end{equation}
where \(\rho \in (0,1)\). Equivalently,
\[
    P(\cdot \mid s)
    =
    \mathcal{N}\left(\rho s, \sigma^2 I_2\right).
\]
The contraction parameter \(\rho < 1\) encourages trajectories to remain concentrated around the origin, while the Gaussian noise introduces stochasticity into the state evolution. In our experiments, we set \(\rho=0.5\) and \(\sigma=0.2\), which produces smooth stochastic trajectories from any sampled initial states. We use a discount factor of \(\gamma=0.9\), consistent with the experimental setting used in MetaWorld.

Under these linear Gaussian dynamics, the Markov chain admits a stationary distribution
\[
    \mathcal{N}\left(
        0,
        \frac{\sigma^2}{1-\rho^2} I_2
    \right).
\]

The two-dimensional state-space construction makes it possible to visualize both the ground-truth value function and the transformer-approximated value function as three-dimensional surface plots.

\paragraph{Kernel and landmark construction.}
We use the same exponential kernel as the one induced by the nonlinear attention activation:
\begin{equation*}
    \kappa(x,y) = \exp\left(\frac{x^\top y}{\delta}\right),
\end{equation*}
where \(\delta>0\) denotes the kernel temperature. In our experiments, we set \(\delta=1\).

To construct the ground-truth value function, we place \(m\) centroid points uniformly on the unit circle:
\begin{equation*}
    c_j =
    \left(
    \cos\left(\frac{2\pi j}{m}\right),
    \sin\left(\frac{2\pi j}{m}\right)
    \right),
    \qquad j=0,\ldots,m-1.
\end{equation*}
In our experiments, we set \(m=8\). This symmetric configuration induces a smooth value landscape, making it well suited for visualization. Figure~\ref{appfig:landmarks} visualizes the centroid points, which are evenly distributed along the unit circle.
\begin{figure}[t]
    \centering
    \includegraphics[width=0.45\columnwidth]{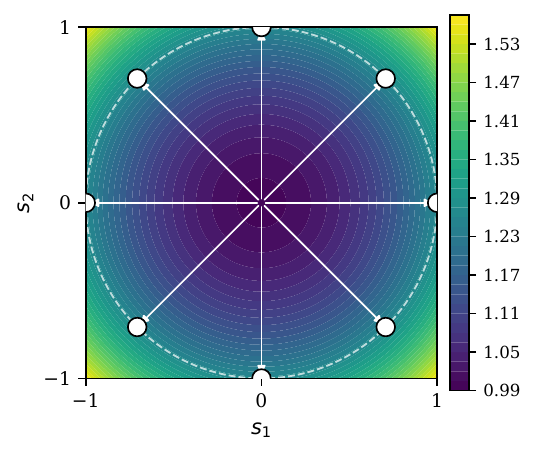}
    \caption{Selected centroid points and the induced state value function.}
    \label{appfig:landmarks}
\end{figure}

\paragraph{Ground-truth value function.}
We define the ground-truth state value function as a uniform mixture of kernel functions centered at the landmark points. Specifically, for every \(s \in [-1,1]^2\),
\begin{equation*}
    V^\star(s)
    =
    \frac{1}{m}
    \sum_{j=0}^{m-1}
    \kappa(c_j,s)
    =
    \frac{1}{m}
    \sum_{j=0}^{m-1}
    \exp\left(\frac{c_j^\top s}{\delta}\right).
    \label{eq:synthetic_vstar}
\end{equation*}
Intuitively, \(V^\star(s)\) measures the aggregate kernel similarity between the state \(s\) and the centroid points. Since the centroids are uniformly distributed along the unit circle, states with larger norms tend to align more strongly with a subset of the centroids. Consequently, the value function increases as the state moves away from the origin, with the largest values attained near the corners of the state space. Figure~\ref{appfig:landmarks} visualizes this effect.

By construction, \(V^\star \in \mathcal{H}_\kappa\), since it is a finite linear
combination of kernel evaluations.

\paragraph{Bellman-consistent reward.}
We derive the reward function from the
desired value function so that the Bellman equation is satisfied exactly. For a
single kernel term, using the Gaussian transition dynamics gives
\begin{align*}
    \mathbb{E}\left[\kappa(c_j,s_{t+1}) \mid s_t=s\right]
    &=
    \mathbb{E}\left[
    \exp\left(\frac{c_j^\top(\rho s+\epsilon_t)}{\delta}\right)
    \right] \\
    &=
    \exp\left(\frac{(\rho c_j)^\top s}{\delta}\right)
    \exp\left(\frac{\|c_j\|^2\sigma^2}{2\delta^2}\right).
\end{align*}
Since the centroids lie on the unit circle, \(\|c_j\|=1\), and therefore the
Gaussian correction term is the same for all landmarks. Denote
\begin{equation*}
    \eta = \exp\left(\frac{\sigma^2}{2\delta^2}\right).
\end{equation*}
Then
\begin{equation*}
    \mathbb{E}\left[V^\star(s_{t+1}) \mid s_t=s\right]
    =
    \frac{\eta}{m}
    \sum_{j=0}^{m-1}
    \kappa(\rho c_j,s).
\end{equation*}
We define the reward function by rearranging the Bellman equation
\(V^\star(s)=R(s)+\gamma\mathbb{E}[V^\star(s_{t+1})\mid s_t=s]\):
\begin{equation}
    R(s)
    =
    \frac{1}{m}
    \sum_{j=0}^{m-1}
    \left[
    \kappa(c_j,s)
    -
    \gamma \eta \kappa(\rho c_j,s)
    \right].
    \label{eq:synthetic_reward}
\end{equation}
Substituting into the Bellman equation gives
\begin{align*}
    R(s) + \gamma \mathbb{E}\left[V^\star(s_{t+1}) \mid s_t=s\right]
    &=
    \frac{1}{m}
    \sum_{j=0}^{m-1}
    \left[
    \kappa(c_j,s)
    -
    \gamma \eta \kappa(\rho c_j,s)
    \right]
    +
    \frac{\gamma \eta}{m}
    \sum_{j=0}^{m-1}
    \kappa(\rho c_j,s) \\
    &=
    \frac{1}{m}
    \sum_{j=0}^{m-1}
    \kappa(c_j,s)
    =
    V^\star(s).
\end{align*}
Thus, the constructed reward function can be exactly calculated for each state and is Bellman-consistent with
\(V^\star\). Moreover, \(R\in\mathcal{H}_\kappa\), since it is also a finite
linear combination of kernel functions.

\paragraph{Transformer Input.}
For each prompt, we uniformly sample a set of context states
\((s_0,s_1,\ldots,s_{n-1})\)
from the synthetic state space. These states are then used to generate the corresponding next states
\((s_0',s_1',\ldots,s_{n-1}')\)
and rewards
\((r_0,r_1,\ldots,r_{n-1})\)
according to the transition dynamics in Equation~\ref{eq:synthetic_transition} and the reward function in Equation~\ref{eq:synthetic_reward}. An arbitrary query state is then appended to the context to form \(Z_0\), which is provided as input to the transformer for value prediction. The transformer performs kernel TD updates through its forward pass, using the context states as kernel centers.

Importantly, the input provided to the transformer consists only of individual transition samples. The context states are uniformly distributed across the state space and are not chosen to match the centroid configuration used to define \(V^\star\). Thus, recovering the underlying state value function requires the transformer to correctly implement the TD update procedure from the sampled transitions.

In the experiment shown in Figure~\ref{fig:syntheic_env}, we uniformly sample
\(32\) context states from the state space together with their corresponding next
states and rewards. We then fix the context and vary the query state over a grid
in \([-1,1]^2\), allowing us to visualize the transformer-predicted value
surface. We use \(30\) transformer layers, consistent with the MetaWorld
experiments. As shown in the results, the transformer accurately recovers the underlying state value function, up to the constant bias discussed above.


\newpage
\section*{NeurIPS Paper Checklist}

\begin{enumerate}

\item {\bf Claims}
    \item[] Question: Do the main claims made in the abstract and introduction accurately reflect the paper's contributions and scope?
    \item[] Answer: \answerYes{}{} 
    \item[] Justification: We explicitly included the contribution and scope.

\item {\bf Limitations}
    \item[] Question: Does the paper discuss the limitations of the work performed by the authors?
    \item[] Answer: \answerYes{} 
    \item[] Justification: We explicitly included in section~\ref{sec:limitations} to discuss about limitations.

\item {\bf Theory assumptions and proofs}
    \item[] Question: For each theoretical result, does the paper provide the full set of assumptions and a complete (and correct) proof?
    \item[] Answer: \answerYes{} 
    \item[] Justification: We provide formal settings and full derivations in section~\ref{sec:theory} and appendix~\ref{app:construction}.

    \item {\bf Experimental result reproducibility}
    \item[] Question: Does the paper fully disclose all the information needed to reproduce the main experimental results of the paper to the extent that it affects the main claims and/or conclusions of the paper (regardless of whether the code and data are provided or not)?
    \item[] Answer: \answerYes{} 
    \item[] Justification: We include experiment details in appendix~\ref{app:hyperparameters} and appendix~\ref{app:synthetic_env}.

\item {\bf Open access to data and code}
    \item[] Question: Does the paper provide open access to the data and code, with sufficient instructions to faithfully reproduce the main experimental results, as described in supplemental material?
    \item[] Answer: \answerYes{} 
    \item[] Justification: We provide the code as supplemental material.

\item {\bf Experimental setting/details}
    \item[] Question: Does the paper specify all the training and test details (e.g., data splits, hyperparameters, how they were chosen, type of optimizer) necessary to understand the results?
    \item[] Answer: \answerYes{} 
    \item[] Justification: We include experiment details in appendix~\ref{app:hyperparameters} and appendix~\ref{app:synthetic_env}.

\item {\bf Experiment statistical significance}
    \item[] Question: Does the paper report error bars suitably and correctly defined or other appropriate information about the statistical significance of the experiments?
    \item[] Answer: \answerYes{} 
    \item[] Justification: We include experiments with random seeds and sufficient stochasticity.

\item {\bf Experiments compute resources}
    \item[] Question: For each experiment, does the paper provide sufficient information on the computer resources (type of compute workers, memory, time of execution) needed to reproduce the experiments?
    \item[] Answer: \answerNo{} 
    \item[] Justification: The work requires only minimal computational resources for training, and we believe it can be reproduced with standard computing capacity.
    
\item {\bf Code of ethics}
    \item[] Question: Does the research conducted in the paper conform, in every respect, with the NeurIPS Code of Ethics \url{https://neurips.cc/public/EthicsGuidelines}?
    \item[] Answer: \answerYes{} 
    \item[] Justification: We comply with the code of ethics.

\item {\bf Broader impacts}
    \item[] Question: Does the paper discuss both potential positive societal impacts and negative societal impacts of the work performed?
    \item[] Answer: \answerNo{} 
    \item[] Justification: The work is purely algorithmic, and to the best of our knowledge, it does not raise any direct societal impact concerns.
    
\item {\bf Safeguards}
    \item[] Question: Does the paper describe safeguards that have been put in place for responsible release of data or models that have a high risk for misuse (e.g., pre-trained language models, image generators, or scraped datasets)?
    \item[] Answer: \answerNo{} 
    \item[] Justification: The work is purely algorithmic, and to the best of our knowledge, it does not raise any direct safety concerns.

\item {\bf Licenses for existing assets}
    \item[] Question: Are the creators or original owners of assets (e.g., code, data, models), used in the paper, properly credited and are the license and terms of use explicitly mentioned and properly respected?
    \item[] Answer: \answerYes{} 
    \item[] Justification: We wrote the code by our own.

\item {\bf New assets}
    \item[] Question: Are new assets introduced in the paper well documented and is the documentation provided alongside the assets?
    \item[] Answer: \answerYes{} 
    \item[] Justification: The code is well structured and includes a README file.

\item {\bf Crowdsourcing and research with human subjects}
    \item[] Question: For crowdsourcing experiments and research with human subjects, does the paper include the full text of instructions given to participants and screenshots, if applicable, as well as details about compensation (if any)? 
    \item[] Answer: \answerNA{} 
    \item[] Justification: The work is purely algorithmic, and to the best of our knowledge, it does not involve any human subjects.

\item {\bf Institutional review board (IRB) approvals or equivalent for research with human subjects}
    \item[] Question: Does the paper describe potential risks incurred by study participants, whether such risks were disclosed to the subjects, and whether Institutional Review Board (IRB) approvals (or an equivalent approval/review based on the requirements of your country or institution) were obtained?
    \item[] Answer: \answerNA{} 
    \item[] Justification: The work is purely algorithmic, and to the best of our knowledge, it does not involve any human subjects.

\item {\bf Declaration of LLM usage}
    \item[] Question: Does the paper describe the usage of LLMs if it is an important, original, or non-standard component of the core methods in this research? Note that if the LLM is used only for writing, editing, or formatting purposes and does \emph{not} impact the core methodology, scientific rigor, or originality of the research, declaration is not required.
    \item[] Answer: \answerNA{} 
    \item[] Justification: LLMs were used only for editing and sentence polishing. The methodology, scientific rigor, and originality of the research are entirely the authors’ own.

\end{enumerate}

\end{document}